\documentclass[accepted]{uai2026} 
\usepackage{macros}
                        
\usepackage[american]{babel}
\usepackage{enumitem} 

\usepackage{natbib} 
\usepackage{mathtools} 
\usepackage{booktabs} 
\usepackage{tikz} 

\title{Causal Discovery in Mixtures of Populations}

\author[1]{\href{mailto:<bijan@dartmouth.edu>?Subject=Causal Discovery in Mixtures of Populations}{Bijan H. S. Mazaheri}{}}
\author[2]{Spencer Gordon}
\author[3]{Yuval Rabani}
\author[4]{Leonard Schulman}
\affil[1]{%
    Thayer School of Engineering\\
    Dartmouth College\\
    Hanover, New Hampshire, USA
}
\affil[2]{%
    Fragment Data Technologies\\
    San Francisco, California, USA
}
\affil[3]{%
    Department of Computer Science\\
    Hebrew Univeristy of Jerusalem\\
    Jerusalem, Israel
  }
\affil[4]{%
    Computing and Mathmematical Sciences\\
    California Institute of Technology\\
    Pasadena, California, USA
  }
  
\begin{document}
\maketitle

\begin{abstract}
  Causal discovery aims to learn causal structures up to certain symmetries. Diverse populations or changing environments give rise to heterogeneous data in the following sense: each population/environment is a ``source'' which idiosyncratically determines the forms of causal effects. From this perspective, the source is a latent common cause for every observed variable. While some methods for causal discovery can work around latent confounding in special cases, a global confounder poses a significant challenge. The only known ways to deal with latent global confounding involve making assumptions that limit structural equations and/or noise functions. We demonstrate that globally confounded causal structures can still be identified with arbitrary structural equations and noise functions, so long as the number of latent classes remains small relative to the size and sparsity of the underlying DAG. The approach relies on agglomerating variables into large-enough matrices of moments, whose ranks directly reveal graphical properties of the causal structure. We also provide a statistical test to test the rank of these matrices.
\end{abstract}

\section{INTRODUCTION}
Many approaches to studying causal systems use structural causal models (SCMs) to graphically model causal relationships in a directed acyclic graph (DAG) \citep{pearl2009causality}.  In an SCM,  $A \rightarrow B$ indicates ``$A$ has a direct causal effect on $B$.'' 

Formally, ``d-separation rules''\footnote{See \citet{pearl2009causality} for a review of d-separation or Appendix~\ref{apx: dsep} for a summary of important results used in this manuscript.} give graphical criteria for statistical independence and dependence under the assumptions of faithfulness (independence implies d-separation) and the causal Markov condition (d-separation implies independence). Constraint-based causal discovery utilizes this connection to uncover information about an unknown causal structure \citep{spirtes2000causation}. 

A critical assumption for many causal discovery algorithms is ``causal sufficiency,'' i.e., no unobserved confounding. In most real-world settings, a causal system spans both observed and latent variables. To address this, approaches such as FCI~\citep{spirtes2001anytime} have been developed to uncover equivalence classes of causal structures and the existance of pairwise confounding. In the parametric setting, more ``pervasive'' confounding has been studied under some parametric restrictions on its influence \citep{frot2019robust, cai2023causal, 10.1093/jrsssb/qkad071}.

A ``flipped'' version of this problem is studied in causal representation learning (CRL) \citep{scholkopf_toward_2021, squires2023linear} and latent variable models \citep{silva2006learning, anandkumar2013learning, xie2020generalized, huang2022latent}, which seek to uncover the structure of \emph{latent} variables using \emph{known} (empty or sparse) structure within the \emph{observed} variables. A question of particular theoretical interest in both problems is that of ``identifiability,'' i.e., the conditions under which a structure is uniquely recoverable.

There has been limited exploration into how known relationships \emph{between latent and observed variables} can be utilized to learn more about a system. The simplest setting is causal discovery under pervasive confounding. Pervasive confounding is generally detrimental to identifiability because it exerts a broad confounding influence on the system. However, this known pervasiveness also means that the system contains significant information about the confounding, which makes its effects easier to learn and remove.

The current approaches to causal discovery in the presence of pervasive confounding are parametric and provide limited insight into the identifiability trade-offs involved. No approaches have yet addressed the issue of a \emph{discrete} universal (i.e., edges to all observable variables) latent confounder. Discrete confounding allows for a simpler quantification of the ``amount'' of unobserved confounding: the number of latent populations or mixture components. Quantifying the relationship between confounding and the resulting identifiability requirements on the observed system is of particular theoretical interest.

Practically, discrete confounding is related to the study of mixture models. Combined data from multiple sources or populations often gives rise to such a setting. In many mixture problems, data is also discrete, such as the study of single-nucleotide polymorphisms (or SNPs) in heterogeneous populations.

\subsection{Problem Statement}
\begin{figure}[h]
    \centering
    \scalebox{.67}{
    \begin {tikzpicture}[-latex ,auto ,node distance =1.5 cm and 1.5 cm ,on grid , semithick, state/.style ={ circle, draw, minimum width =.5 cm}, cstate/.style ={ circle, draw, minimum width =.5 cm, ultra thick}]
            \filldraw[color=blue, fill=blue!5, very thick](-2.1,-3.6) rectangle (2.1, .6)
            node[right] {$\G'$};
            \filldraw[color=red, fill=red!5, very thick](-2,-3.5) rectangle (2, -1)
            node[right] {$\G$};
            \node[state, color=blue, dashed] (U1) {$U$};
            \node[state] (X1)  [below left =of U1]{$V_1$};
            \node[state] (X2) [right =of X1]{$V_2$};
            \node[state] (X3) [right  =of X2] {$V_3$};
            \node[state] (X4) [below  =of X1] {$V_4$};
            \node[state] (X5) [right  =of X4] {$V_5$};
            \node[state] (X6) [right  =of X5] {$V_6$};
            \path[color=blue] (U1) edge (X1);
            \path[color=blue] (U1) edge (X2);
            \path[color=blue] (U1) edge (X3);
            \path[color=blue] (U1) edge (X4);
            \path[color=blue] (U1) edge[bend left =30] (X5);
            \path[color=blue] (U1) edge[bend left = 10] (X6);
            \path (X1) edge (X2);
            \path (X2) edge (X6);
            \path (X4) edge (X1);
            \path (X4) edge (X5);
            \path (X5) edge (X6);
    \end{tikzpicture}
    }
    \caption{The goal is to learn the graph structure $\G$ \emph{without} observing $U$. } \label{fig: example_setup} 
  \end{figure}

Suppose we augment a DAG $\G= (\vec{V}, \vec{E})$ with an unobserved $U$ that has $k$ ``latent classes'', each of which exerts a distinctive signal on all the observed variables $\vec{V}$. The result is a mixture model whose DAG, 
$\G' = (\vec{V} \cup \{U\}, \vec{E}')$ includes additional arrows from $U$ to every $V \in \vec{V}$.  See Figure~\ref{fig: example_setup} for an example. We will refer to $\G$ as the observed sub-graph, and we will use $\Delta := \max_{V \in \vec{V}} \deg^{\G}(V)$ to denote its maximum (in plus out) degree.

The goal is to uncover the observed sub-graph's structure $\G$ up to a Markov equivalence class (MEC) (i.e., a CP-DAG) using statistics gained from the ``observed marginal distribution,'' i.e., $\Pr(\vec{V})$ marginalized over $U$.  Notice that $U$ confounds all pairs of variables in $\vec{V}$, which requires it to be in \emph{every} separating set.  This means that the correspondence between conditional independence and d-separation is no longer adequate to make any deductions about causal structure.

\paragraph{Assumptions}
We will assume (1) that the distribution is faithful\footnote{The precise assumption is a slight extension of faithfulness to the mixture setting, discussed later.} with respect to $\G'$, (2) that $U$ is discrete with a known number of latent classes $k$, and (3) that there are no latent variables other than $U$ (i.e., causal sufficiency holds for $\G'$).  Throughout, all variables are discrete: the observables $\vec{V}$ are categorical (we will focus on binary $\vec{V}$) and $U$ is also categorical with $k$ classes. $k$ therefore represents the ``complexity'' of unobserved confounding, as more latent classes are capable of exerting ``more complex'' signals. Since our deductions rely on each value of $U$ having distinct effects on the observables, the most difficult case arises for Bernoulli observable variables expressed by a single parameter.

We make no parametric assumptions (e.g, linearity of the structural equations or Gaussian noise) on the observed portion of the graph. The only mild parametric assumption is the categorical nature of $U$. Our approach does not require Bernoulli observables, but is built around this especially difficult case to achieve more general identifiability results.

\subsection{Contributions}
\paragraph{Tests of Rank.} Conditional independence testing is considered a crucial component of constraint-based methods for non-parametric causal discovery \citep{spirtes2000causation, squires2023causal}. By studying a problem for which there are \emph{no conditional independencies}, we show that conditional independence tests are a special case of a broader set of tests that may be modified based on parametric assumptions about unobserved confounding. These modifications can be made while avoiding parametric assumptions on the observable variables. In this paper, we use a test for matrix rank and develop a hypothesis test for it based on the work of \citet{ratsimalahelo2001rank}. We demonstrate that this approach is superior to thresholding singular values, as in \citet{anandkumar2012learning}, and note that it may be useful in other settings that use matrix rank, such as \citet{squires2023linear}. The test can be accessed using \texttt{pip install probrank}.

\paragraph{Agglomeration in Structure Learning.} Tests of matrix rank are usually used in ``large alphabet'' settings, where random variables can be reduced to categorical variables with large cardinality \citep{anandkumar2010high, anandkumar2012method, anandkumar2012learning}. To extend the use of these tests to a ``small alphabet'' setting (i.e., Bernoulli observables), we show how to group or ``agglomerate'' simple variables into ones with larger alphabets. This process is nontrivial, since the properties of these agglomerated tests must then be disentangled to make conclusions about the original Bernoulli random variables. As such, the paper makes significant theoretical contributions to the conditional independence (and matrix rank) properties of agglomerated variables, which may have use in areas of causal abstraction \citep{beckers_abstracting_2019, beckers_approximate_2020} and causal representation learning \citep{scholkopf_toward_2021}.

\paragraph{Algorithmic and Theoretical Contributions.} Our main result is, to our knowledge, the first algorithm for identifying causal structures in this mixture setting without parametric assumptions on the observed variables. The correctness of this algorithm and the uniqueness of its output provide sufficient conditions for identifying causal structures from data across multiple populations.

The identification of a causal structure (or equivalence class) is based on the ability to eliminate candidate structures based on the observed properties of the probability distribution. A larger $k$ therefore reduces identifiability by increasing the support of valid probability distributions. For example, $k \geq 2^{\abs{\vec{V}}}$ is sufficient to model \emph{any} marginal distribution on Bernoulli variables $\vec{V}$, even with a fully underpowered $\vec{E} = \emptyset$. To see this, assign one value of $u$ to each of the $2^{\abs{\vec{V}}}$ possible assignments to the $\abs{\vec{V}}$ observed binary variables. By changing the relative probabilities of $U=u$, we may specify any marginal distribution. Since any distribution can be generated in this way (even when the graph is empty), such a large $k$ clearly prohibits identifiability.

Identification of a causal structure also depends on the number of observed variables $\abs{\vec{V}}$ and the sparsity of $\G$. A larger $\abs{\vec{V}}$ provides more information about $U$, and a sparser $\G$ further restricts the possible distributions generated (e.g., a complete graph may generate any marginal distribution even with $k=1$). A key information-theoretic question is how the identifiability requirements for $\G$ scale relative to $k$, $\abs{\vec{V}}$, and the maximum degree $\Delta$. To this end, the conditions under which our algorithm recovers a unique (and correct) MEC---established later via Theorem~\ref{lem: phase II runtime and vertex requirement}---also yield the following identifiability result.

\begin{corollary}\label{thm: ident}
 Consider $\G = (\vec{V}, \vec{E})$ with mixture source $U\in \{1, \ldots, k\}$ and degree bound $\Delta$.  If \[\abs{\vec{V}} \geq (\Delta^3 + 2 \Delta^2 + 4 \Delta + 2) \lceil \lg(k + 1) \rceil + 2 \Delta^2 + 2\Delta^3,\] then $\G$ is generically\footnote{Generic identification means that the unidentifiable cases have Lebesgue measure 0 in the space of parameters.} identifiable up to its MEC.  
\end{corollary}

In practice, identifiability often occurs in far smaller graphs, as later experiments with $k=2$ and only $7$ vertices demonstrate. While loose, this is the best and only known bound to guarantee identifiability.

We achieve this identifiability result by developing an algorithm that repeatedly identifies a mixture of $k$ product distributions (which we call $k$-MixProd) using an oracle that takes time $\tau$. We also use an oracle that can solve for non-negative rank in time $\rho$. Our algorithm runs in time $\abs{\vec V}^{\mathcal{O}(\Delta^2 \log(k))}\rho +\mathcal{O}((\Delta^2 \lg(k)) + \abs{\vec{E}})k 2^{\Delta^2}) \tau$.  Both Corollary~\ref{thm: ident} and the associated runtime are a consequence of the algorithm and its requirements for its success, as outlined in Theorem~\ref{lem: phase II runtime and vertex requirement}.
See Appendix~\ref{apx: further runtime discussion} for further discussion on the runtime of $\rho$ and $\tau$.

\subsection{Related Works}
\paragraph{Causal discovery.}
The PC algorithm was the first causal discovery algorithm to use conditional independence, as outlined by \citet{spirtes2000causation}. Many causal discovery algorithms have since been developed, summarized in \citet{squires2023causal}. A number of algorithms address the presence of latent confounding using one of two assumptions that ``limit'' the confounding influence: (1) low out-degree latents or (2) parametric limitations in the random variables.

The first type of assumption involves limiting the \emph{out-degree} of latent confounding. For example, the FCI algorithm can detect the presence of unobserved confounders that act on \emph{only two} observed variables \citep{spirtes1993discovery, spirtes2001anytime}.  \citet{richardson2002ancestral}'s seminal work introduced ancestral graphs for the general study of this setting. It is important to note that ancestral graphs are larger equivalence classes than MECs. As such, these approaches recover the structure at a coarser resolution, and do so under assumptions that do not hold in the mixture setting.

\paragraph{Global Confounding.}
Recent work has studied ``global'' or ``pervasive'' confounding \citep{gordon2023causal, frot2019robust, 10.1093/jrsssb/qkad071}. Pervasive confounders affect all observable variables in a system and are often observed when data are gathered over a large area or over a long period of time. Such confounding is graphically modeled as a single unobserved variable that links every observable variable (see Figure 1). This structure d-connects all possible pairs of vertices, making all variables statistically dependent even after conditioning on any other variables, which rules out the use of the PC algorithm. 

Pervasively confounded DAGs cannot be learned using only conditional independence tests. Instead, these settings require parametric assumptions that restrict the space of probability distributions. For example, \citet{frot2019robust} was able to show superior performance in settings with large-degree confounders with linear relationships and additive sub-Gaussian or elliptical noise.  Other settings include linear structural equations with non-Gaussian additive noise \citep{cai2023causal} and non-linear structural equations within a finite-dimensional Hilbert space  \citep{10.1093/jrsssb/qkad071}.

The setting studied in this paper involves a discrete unobserved pervasive confounder that imposes a ``mixture'' or ``latent class''. Methods designed for continuous latent confounders are ill-equipped to exploit the special properties of mixture models. Furthermore, the methods rely on parametric assumptions about the observed variables. Since our approach requires no parametric assumptions on the observed variables, it provides a clearer picture of the identifiability conditions in a more general setting.

\paragraph{k-MixProd.} Almost all research on mixtures with categorical data leverages the assumption of mutually independent observed variables within the source distributions \citep{FM99, CGG01, CR08, FOS08, RSS14, ChenMoitra19, gordon2021source, gordon2024identification} (see also the earlier seminal work of \citet{KMRRSS94}). We call this the $k$-MixProd problem (where $k$ denotes the number of latent classes), whose identifiability is discussed in Appendix~\ref{apx: kmixprod}.

$k$-MixProd assumes that the causal model within each source is an empty graph and aims to learn the conditional distributions within each mixture component. Early work by \citet{anandkumar2012learning} and \citet{gordon2023causal} broadened this class of independence assumptions, exploring Markov random fields and Bayesian networks, respectively. A key assumption of \citet{gordon2023causal} and $k$-MixProd is that the Bayesian DAG structure within each component of the mixture is known. This paper will develop an algorithm to recover this structure, thereby implicitly providing conditions for the identifiability of Bayesian DAG mixture models with unknown structure.

\paragraph{Latent Variable Models.}
A related and complementary problem in causal discovery is learning a causal graph over many latent variables from their observable children. Many classic approaches to this problem use ``rank constraints,'' which are rank deficiencies on covariance matrices between sets of variables \citep{silva2006learning, anandkumar2013learning, xie2020generalized}. This has been extended to hierarchical latent variable models \citep{huang2022latent}. Structure-based approaches have also been used to discover the presence and cardinality of hidden variables \citep{elidan_discovering_2000, elidan_learning_2005}, though these target a different setting from ours. Causal representation learning has also used rank deficiencies \citep{squires2023linear}. In fact, a number of very recent works have employed tests of rank to learn structures in latent variables under a ``purity'' assumption, in which there are no dependencies between the observed random variables \citep{chen_learning_2024, kong_learning_2024,lee_theoretical_2026}---equivalently, the mutual-independence assumption of $k$-MixProd, in which the observed variables are independent given the latent source. The hypothesis test that we develop may be useful for all of these problems. In addition, the algorithmic tools we develop to handle dependencies among observed variables may help move these problems beyond the assumptions of ``purity.''

\paragraph{Network-level Differences.} Causal discovery becomes more difficult when allowing for network-level differences within the mixture components, e.g., \citet{saeed2020causal} and \citet{varici2024separability} as well as \citet{strobl2023causal} for mixtures of cyclic DAGs. We focus on a shared graph structure, allowing for significantly fewer parametric assumptions.

\section{PRELIMINARIES}
\paragraph{Notation.}
We will use the capital Latin alphabet to denote random variables, which are vertices on our DAG. When referring to sets of these variables, we will use bolded font, e.g. $\vec{X} = \{X_1,  X_2, \ldots\}$. To refer to components of a graph, we will use the following operators:
$\Pa(V), \Ch(V)$ will refer to the parents and children of $V$. $\An(V), \De(V)$ will refer to the ancestors and descendants of $V$.  $\An(V) \cup \{V\}$ and $\De(V) \cup \{V\}$ are denoted using $\iAn(V), \iDe(V)$ respectively. $\Mb(V) \coloneqq \Pa(V) \cup \Ch(V) \cup \Pa(\Ch(V)) \setminus \{V\}$ will refer to the Markov boundary (i.e. the unique minimal Markov blanket) of $V$ \citep{pearl2009causality}. $\Nb_\ell(V)$ refers to the distance $\ell$ neighborhood of $V$. For all of these sets, we will use a bar to denote the complement, e.g., $\overline{\vec{X}} = \vec{V} \setminus \vec{X}$.

As these operators act on graphs,  they can specify the graph structure being used in the superscript, e.g. $\Pa^\G(V)$.  Unless otherwise specified,  operators should be assumed to apply to the observed subgraph $\G$ and not $\G'$, i.e., $U \not \in \Pa^\G(V)$. We will also occasionally write tuples to indicate the intersection of the sets for two vertices, e.g.  $\Ch(V, W) = \Ch(V) \cap \Ch(W)$. Finally, these operators can also act on sets to indicate the union of the operation, e.g.
\begin{equation}
\Pa(\vec{X}) = \bigcup_{X \in \vec{X}} \Pa(X) \setminus \vec{X}.
\end{equation}
We make use of lowercase letters to denote assignments, for example, $x$ will denote the assignment $X = x$. This can also be applied to operators,  e.g. $\mb(V)$ denotes an assignment to $\Mb(V)$. Assignments can be specified by a set of assignments in the subscript, e.g. $\mb_{\vec{c}}(V)$ obtains assignments for $\Mb(V) \subseteq \vec{C}$ from the assignments of $\vec{c}$ to $\vec{C}$. $\nnrank(\mat{M})$ denotes the nonnegative rank of matrix $\mat{M}$.

\paragraph{Constraint-Based Causal Discovery.}
When two variables are found to be conditionally independent given another set, e.g. $A \indep B \given \vec{C}$, we say that $\vec{C}$ is a ``separating set'' for $A,B$, which acts as a witness for the absence of an edge between them. This critical observation allows the discovery of a ``skeleton,'' i.e., a causal graph containing all causal adjacencies with no orientations. In addition, separating sets also allow in-degree $>2$ vertices with non-adjacent parents to have their parental edges oriented. These are called v-structures, unshielded colliders, or immoralities.

These partial DAG structures become ``completed partial DAGs,'' or CP-DAGs \citep{andersson1997characterization}, by further orienting edges according to the ``Meek rules'' \citep{meek95} that enforce acyclicity and no additional unshielded colliders. CP-DAGs represent a set of possible DAG structures that agree with the observed conditional independence, which is called a Markov equivalence class (MEC).

\paragraph{Outline of Algorithm.}
Our approach is based on the PC-algorithm \citep{spirtes2000causation}, which works in two phases. Phase I will begin with a complete graph and remove edges between variables when we find evidence of non-adjacency (using rank constraints). This phase will only remove edges between groupings of variables, so its termination will not guarantee that we have discovered all possible non-adjacencies. Instead, a provably small subset of the graph will have \emph{false positive} edges.

In Phase II, we will use the structure we have uncovered to induce instances of $k$-mixtures of products, a famous and long-studied problem \citep{gordon2021source, gordon2024identification, allman2009identifiability}, which we will abbreviate $k$-MixProd. An oracle solver for this problem will then identify the joint probability distribution between subsets of $\vec{V}$ and the latent class $U$.  Access to this joint probability distribution allows the rest of the structure to be resolved using conditional independence tests on distributions that are conditioned on $U$.

Phase III mirrors the last phase of the PC algorithm: identifying immoralities using non-adjacencies and separating sets, and then propagating orientations according to Meek rules. This phase is identical to the PC algorithm, so its details are omitted from this manuscript.

\section{RANK TESTS} \label{sec:rank tests}
This section will introduce ``rank tests'' which will serve as a replacement for conditional independence tests as a test for d-separation or d-connectedness. These serve as a generalization of rank-deficiencies of covariance matrices noted in \citet{spirtes2000causation}. The use of rank was also mentioned in \citet{anandkumar2012method}, but not formalized. Rank constraints induced by hidden variables of restricted cardinality have also been studied by \citet{zjawin_restricted_2021}.

\subsection{$k$-Mixture Independence}
To determine non-adjacency, we will leverage the signal $U$ leaves on the marginal probability distributions of variables that are independent conditional on $U$.  First, we interpret the marginal probability distribution as a matrix.
\begin{definition} \label{def: prob mat}
	Given two discrete variables $X, Y \in \vec{V}$ with $\abs{X} = n$ and $\abs{Y} = m$, define the ``probability matrix''  $\mat{M}[X, Y] \in [0, 1]^{n \times m}$ to be
\begin{equation}
\mat{M}[X, Y]_{x, y} := \Pr(x, y),
\end{equation}
where $x \in \{1, \ldots, n\}$ and $y \in \{1, \ldots, m\}$. Similarly, for $\vec{C} \subseteq \vec{V}$, define
\begin{equation}
\mat{M}[X, Y \given \vec{c}]_{x, y} := \Pr(x, y \given \vec{c}).
\end{equation}
\end{definition}
These matrices have ranks that give valuable information about the graphical structure under latent class confounding. Our algorithm will use the following fact.
\begin{lemma}[Rank Test]\label{lem:rank_test}
For $V_i, V_j$ with cardinality $>k$ confounded by $U$ in $\G'$, $V_i \indep^{\G}_d V_j \given \vec{C}$ if and only if (generically) $\nnrank(\mat{M}[V_i, V_j \given \vec{C}]) \leq k$.
\end{lemma}

The cardinality requirement $\abs{V_i}, \abs{V_j} > k$ is needed only to ensure that d-connectedness forces $\nnrank > k$ (the reverse direction, via Lemma~\ref{lem:nnrank_dependence}); the forward implication, that d-separation yields $\nnrank \leq k$, holds for any cardinality.

To prove this Lemma we must show both directions. First, we observe that we can decompose $\mat{M}[X, Y \given \vec{c}]$ as follows:
\begin{equation} \label{eq: decompose M}
\mat{M}[X, Y \given \vec c] = \sum_{u} \Pr(u) \mat{M}[X, Y \given \vec{c}, u].
\end{equation}
$X \indep Y \given \vec{C}, U$, so $\mat{M}[X, Y \given \vec{c}, u]$ can be written as the outer product of two vectors describing the probabilities of each variable.  Therefore, we conclude that $\nnrank(\mat{M}[X, Y \given \vec{c}, u]) = 1$. If $U \in [k]$, then summing $k$ such terms gives $\nnrank(\mat{M}[X, Y \given \vec{c}]) \leq k$.

Having shown that d-separation in $\G$ upper bounds the rank of probability matrices, we now seek a lower bound on the rank in the case of d-connectedness.  This will require a ``faithfulness-like'' assumption that the $\mat{M}[X, Y \given \vec{c}, u]$ terms in Equation~\ref{eq: decompose M} are linearly independent, resulting in an overall $\nnrank{\mat{M}(X, Y \given \vec{c})} > k$ when $X$ and $Y$ are not independent conditioned on $U$. Lemma~\ref{lem:nnrank_dependence} shows that this condition holds generically.

\begin{lemma}\label{lem:nnrank_dependence}
Consider a mixture of Bayesian network distributions, each of which is faithful to $\G$.  If $X \not\indep^{\G}_d Y \given \vec{C}$ and $\abs{X}=n$,  $\abs{Y} = m$ with $n,m > k$, then for all $\vec{c}$, $\nnrank(\mat{M}[X, Y \given \vec{c}]) > k$ with full Lebesgue measure.
\end{lemma}
The proof of Lemma~\ref{lem:nnrank_dependence} (see Appendix~\ref{apx: deferred proofs}) involves applying faithfulness to each summand of Equation~\ref{eq: decompose M}.
To make these rank tests practical, Appendix~\ref{sec: hyp test} develops a hypothesis test for the rank of a noisy matrix of estimated probabilities, which we later test in Section~\ref{sec: experiments}.

\section{ALGORITHM}
We will now outline the first two phases of our algorithm, leaving out the final phase of orienting edges with respect to Meek's rules. The first phase involves agglomerating sets of variables and applying rank tests. We then analyze the output $\G_1$, which has removed \emph{most but not all} of the missing edges. We define FP edges as these ``leftovers'' and show that they are contained in a bounded-size subset. This fact allows us to resolve all of the non-adjacencies in $\G$ by setting up instances of $k$-MixProd in Phase II.
\subsection{Phase I: Agglomerated Rank Tests}
Lemma~\ref{lem:rank_test} allows for a simple generalization of the PC algorithm provided that our probability matrix $\mat{M}[X, Y]$ is at least $k+1$ by $k+1$.  Unfortunately, categorical variables ranging over smaller alphabets (such as the binary alphabets addressed by this paper) do not contain sufficient information to detect non-adjacency in cases of larger $k$. We resolve this problem by the agglomeration of $\lg(k+1)$-cardinality sets of small-alphabet (binary) variables into supervariables of larger carinality (i.e., $2^{\lg(k+1)} = k + 1$).

\begin{definition}
Consider DAG $\G=(\vec{V}, \vec{E})$, $V_i, V_j \in \vec{V}$, and sets $\vec{S}_i, \vec{S}_j \subseteq \vec{V} \setminus \{V_i, V_j\}$. We call the ordered pair $(\vec{S}^+_i, \vec{S}^+_j) = (\vec{S}_i \cup \{V_i\},\vec{S}_j \cup \{V_j\})$ an \textbf{independence preserving agglomeration (IPA)} of $(V_i, V_j)$ if, for some \textbf{IPA conditioning set} $\vec{C} \subset \vec{V}$, \[\vec{S}^+_i\indep^{\G}_d \vec{S}^+_j \given \vec{C}.\]
\end{definition}
The creation of supervariables allows us to use conditional rank tests in place of conditional independence tests. This leads to a modified version of the PC algorithm that searches over pairs of supervariable agglomerations instead of pairs of vertices, given in Algorithm~\ref{alg:phase1}.

\begin{algorithm} 
    \caption{Phase I}\label{alg:phase1}
    \DontPrintSemicolon
    \KwIn{The marginal probability distribution $\Pr(\vec{V})$, marginalized over $U$.}
 	\KwOut{An undirected graph $\G_1 = (\vec{V}, \vec{E}_1)$ and a separating set $\vec{C}_{ij}$ for each detected non-adjacency.}
 	
       	Begin with a complete undirected graph $\G_1=(\vec{V}, \vec{V} \times \vec{V})$ and $d_{\max} \gets \abs{\vec{V}} - 1$.
       	
       	\For{$\ell=0$ to $\ell=d_{\max}$}{
       		\For{$\vec{C} \subset \vec{V}$ and $\abs{\vec{C}} = \ell$}{
       			\For{ $\vec{S}, \vec{S}' \subseteq \vec{V} \setminus \vec{C}$, with $\abs{\vec{S}} = \abs{\vec{S}'} = \lceil \lg(k) \rceil + 1$}{
       				\If{arbitrary assignment $\vec{c}$ has $\nnrank(\mat{M}[\vec{S}, \vec{S}' \given \vec{c}])\leq k$}{
       					Remove edges between $\vec{S}$ and $\vec{S}'$ in $\G_1$. 
   
						$\vec{C}_{i, j} \gets \vec{C}$ for each $V_i \in \vec{S}, V_j \in \vec{S}'$
       					
       					Update $\dmax$ to the max degree of $\G_1$.
       				}
       			}
       		}
       	}
       	
\end{algorithm}

\begin{lemma}\label{lem: runtime phase1}
 Algorithm~\ref{alg:phase1} utilizes $\abs{\vec V}^{\mathcal{O}(\Delta^2 \log(k))}$ non-negative rank tests.
\end{lemma}
\begin{proof}
Lemma~\ref{lem: delta2 sepset size bound} (see Appendix~\ref{apx: dsep}) tells us that the maximum size of a separating set is $\alpha := (\lceil \lg(k) \rceil + 1) \Delta^2$, so we need to check $\binom{\abs{\vec{V}}}{\alpha} + \binom{\abs{\vec{V}}}{\alpha -1} + \ldots + \binom{\abs{\vec{V}}}{1}$ possible separating sets, which is $\abs{\vec V}^{\mathcal{O}(\Delta^2 \log(k))}$.  We must iterate over all possible supervariables for each separating sets,  which is upper bounded by $\binom{\abs{\vec{V}}}{2( \lceil \lg(k) \rceil + 1)}$, which is $\abs{\vec V}^{\mathcal{O}(\log(k))}$. 
\end{proof}
The exponent of $\mathcal{O}(\Delta^2 \log(k))$ in Lemma~\ref{lem: runtime phase1} can be replaced with $\abs{\vec{V}}$, a trivial maximum size of a separating set.
\subsubsection{FP Edges}
Phase I of our algorithm removes an edge between two non-adjacent variables through a rank test so long as there exists an IPA for the non-adjacency.  Not all non-adjacencies will contain an IPA, so the adjacency graph $\G_1$ contains a \emph{superset} of the true adjacencies.
\begin{definition} $\vec{E}_1 \setminus \vec{E}$ are \textbf{false positive (FP)} edges.
\end{definition}

To see why FP edges exist, it helps to isolate the vertices that lie below both endpoints.
\begin{definition}
Let $\vec{D}_{ij} := \De(V_i, V_j)$ be the common descendants of $V_i$ and $V_j$, and let $\vec{A}_{ij} := \vec{V} \setminus \vec{D}_{ij}$ be the remaining vertices (which include $V_i$ and $V_j$).
\end{definition}
Conditioning on a common child of $V_i$ and $V_j$ (any descendant of one) opens the collider $V_i \rightarrow \cdot \leftarrow V_j$, so such a vertex can never belong to a separating set. It cannot be agglomerated into $\vec{S}^+_i$ or $\vec{S}^+_j$ either: separating a descendant of a common child from the opposite endpoint would require conditioning on an ancestor of that ``collider descendant,'' which would open up the collider path between $V_i$ and $V_j$. Consequently, if too many vertices are common children or a descendant of one, no IPA exists, and an FP edge remains after Phase I. Figure~\ref{fig:example_forbidden} shows an extreme case, in which $\vec{A}_{ij} = \{V_i, V_j\}$ leaves nothing to agglomerate.

\begin{figure}[h]
    \centering
    \scalebox{.57}{
    \begin {tikzpicture}[-latex ,auto ,node distance =1.5 cm and 1.5 cm ,on grid , ultra thick, state/.style ={ circle, draw, minimum width =.5 cm}, cstate/.style ={ circle, draw, minimum width =.5 cm, ultra thick}]
    			\filldraw[color=red, fill=red!5, very thick](-2.4,-4) rectangle (4, -.8)
            node[above left] {$\vec{D}_{ij}$};
            \node[state] (X) {$V_i$};
            \node[state] (Y) [right  =of X] {$V_j$};
            \node[state,] (V1) [below = of X]{$V_1$};
             \node[state,] (V2) [below = of Y]{$V_2$};
            \node[rectangle, draw, minimum width =1.5 cm, minimum height = 1.5cm,] (Fb1) [below left = 2 cm of V1] {$\De(V_1)$};
             \node[rectangle, draw, minimum width =1.5 cm, minimum height = 1.5cm,] (Fb2) [below right = 2 cm of V2] {$\De(V_2)$};
            \path (X) edge (V1) (Y) edge (V1) (V1) edge (Fb1);
            \path (X) edge (V2) (Y) edge (V2)(V2) edge (Fb2);
    \end{tikzpicture}
    }
    \caption{An FP edge remains between $V_i$ and $V_j$ after Phase I when every other vertex is a common descendant $\vec{D}_{ij}$ (here, a common child or one of its descendants). Although $V_i$ and $V_j$ are d-separated by $\vec{C} = \emptyset$, no IPA can be formed: the only vertices left to agglomerate lie in $\vec{D}_{ij}$, and conditioning on any of them would open a collider. Equivalently, $\vec{A}_{ij} = \{V_i, V_j\}$ is too small. The population variable $U$ is omitted to avoid clutter.} \label{fig:example_forbidden} 
  \end{figure}

This illustrates that FP edges can occur for pairs of vertices with too many common descendants, leaving too few vertices in $\vec{A}_{ij}$ to form an IPA (shown in Figure~\ref{fig:example_forbidden}).

\subsubsection{Using non-descendants to form IPAs}
It turns out that the majority of pairs of vertices do not have many common descendants and can be agglomerated. To show this, we take $\vec{S}_i, \vec{S}_j \subseteq \vec{A}_{ij}$. The following lemma shows this restriction is \emph{safe}: whenever the two super-variables can be separated at all, they can be separated using a conditioning set drawn entirely from $\vec{A}_{ij}$. Confining the conditions of existance to $\vec{A}_{ij}$ therefore never costs us a witness of non-adjacency, and---since $\vec{A}_{ij}$ excludes every common descendant---such a witness never conditions on an open collider.

\begin{lemma} \label{lem: Aij safe}
Let $V_i, V_j$ be non-adjacent and let $\vec{S}_i, \vec{S}_j \subseteq \vec{A}_{ij}$. If $\vec{S}^+_i$ and $\vec{S}^+_j$ are d-separated by any set, then they are d-separated by some $\vec{C} \subseteq \vec{A}_{ij}$; in particular, $\vec{C}$ contains no common descendant of $V_i, V_j$, so $(\vec{S}^+_i, \vec{S}^+_j)$ is an IPA.
\end{lemma}
\begin{proof}
An ancestor of a non-common-descendant is again a non-common-descendant, so $\vec{A}_{ij}$ is ancestrally closed; hence $\iAn(\vec{S}^+_i \cup \vec{S}^+_j) \subseteq \vec{A}_{ij}$. If $\vec{S}^+_i$ and $\vec{S}^+_j$ are d-separated by some set, they are separated in the moral graph $(\G[\iAn(\vec{S}^+_i \cup \vec{S}^+_j)])^{(m)}$ by a set $\vec{C} \subseteq \iAn(\vec{S}^+_i \cup \vec{S}^+_j)$; by Lemma~\ref{lem: moral graph dsep} this $\vec{C}$ is a separating set in $\G$. Since $\vec{C} \subseteq \iAn(\vec{S}^+_i \cup \vec{S}^+_j) \subseteq \vec{A}_{ij} = \vec{V} \setminus \vec{D}_{ij}$, it contains no common descendant of $V_i, V_j$.
\end{proof}

Lemma~\ref{lem: Aij safe} shows that working inside $\vec{A}_{ij}$ never removes our ability to separate, so the sole remaining obstruction to forming an IPA is having too few vertices in $\vec{A}_{ij}$. Lemma~\ref{lem: sepset exists if we have enough vertices} quantifies how many are needed, and since $\vec{A}_{ij}$ contains the non-descendants of \emph{both} $V_i$ and $V_j$, it is too small only when both endpoints have few non-descendants---that is, when both lie in the ``early vertices'' $\vec{H}$ introduced next.

\begin{lemma} \label{lem: sepset exists if we have enough vertices}
An IPA $\vec{S}_i^+, \vec{S}_j^+$ for $V_i, V_j \in \vec V$ exists so long as $\abs{\vec{A}_{ij}} \geq (2 + \Delta^2)\lceil \lg(k + 1) \rceil - 2$.
\end{lemma}
To utilize Lemma~\ref{lem: sepset exists if we have enough vertices}, we define a special set of vertices, which has an upper-bounded cardinality.
\begin{definition} \label{def: early vertices}
 We define the early vertices,
\[
\vec{H} := \{V \in \vec{V} \text{ s.t. } \abs{
\overline{\De}(V)} < (2 + \Delta^2)\lceil \lg(k+1) \rceil  - 2\}.
\]
\end{definition}
\begin{lemma}\label{obs: size of h and degree}
$\abs{\vec{H}} \leq (2 + \Delta^2)\lceil \lg(k+1) \rceil  - 2$.
\end{lemma}

This definition is carefully crafted so that any pair of vertices with at least one vertex out of $\vec{H}$ has enough non-descendants to augment both vertices to an agglomeration set of at least $\lceil \lg(k + 2) \rceil$ variables. Such a witness to the non-adjacency of these pairs means that they do not have FP edges.
\begin{lemma} \label{lem: all fp in H}
After Phase I (Algorithm~\ref{alg:phase1}),  all false positive edges lie within the early vertices, i.e., $\vec E_1 \setminus \vec E \subseteq \vec H \times \vec H$.
\end{lemma}

\subsection{Phase II: Correct FP Edges}
Recall that the marginal probability distribution cannot use independence tests to discover non-adjacency because $U$ confounds all pairs of vertices.  An important observation is that the within-source distribution $\Pr(\vec{V} \given u)$ would not suffer from this limitation because it would allow us to query conditional independence that conditions on the unobserved $U$. 

Phase II will make use of this observation by selecting subsets of variables $\vec{T} \subseteq \vec{V}$ on which to obtain $\Pr(\vec{T} \given u)$ using techniques from discrete mixture models. We will then apply regular conditional independence checks on the recovered $\Pr(\vec{T} \given u)$ to detect FP edges.  We will use a separate $\vec{T}_{ij} \ni V_i, V_j$ agglomeration to verify each edge $(V_i, V_j) \in \vec{E}_1$, though this process can likely be optimized.

We will use a lemma by \cite{allman2009identifiability}, which is a direct consequence of a result by \cite{kruskal1977three} on the uniqueness of decomposing order-3 tensors into rank one components.
\begin{lemma}[\cite{allman2009identifiability}] \label{lemma: kruskal}
Consider the discrete mixture source $U \in \{1, \ldots, k\}$ and discrete variables $X_1, X_2, X_3$ with cardinality $\kappa_1, \kappa_2, \kappa_3$ that are mutually independent given $U$ (i.e., $X_a \indep X_b \given U$ for all $a \neq b$).  The mixture is generically identifiable (with full Lebesgue measure on the parameter space) if
\[\min(\kappa_1, k) + \min(\kappa_2, k) + \min(\kappa_3, k) \geq 2k + 2.\]
\end{lemma}
This result has been used extensively to uncover model parameters when three observables are known to be conditionally independent on an unknown mixing variable \citep{anandkumar2014tensor}. The conditions for identifiability are therefore quite mild --- Phase I only needs to uncover enough sparsity to d-separate three sufficiently large independent agglomerations, one of which will be $\vec{T}_{ij}$ with the intention of verifying or removing $V_i \rightarrow V_j$.  Conveniently, the constrained nature of our FP edges means that the graph $\G_1$ is sufficiently sparse to allow such constructions.

$\vec{T}_{ij}$ must be designed to include enough information to discover a non-adjacency between $V_i, V_j$. In other words, we need to ensure that, if $V_i$ and $V_j$ are \emph{not} adjacent, then $\vec{T}_{ij}$ contains a separating set $\vec{C} \subset \vec{T}_{ij}$ such that $V_i \indep^{\G}_d V_j \given \vec{C}$.  Ensuring that $\vec{T}_{ij}$ also has the distance-1 neighborhoods of $V_i, V_j$ is enough to satisfy this requirement.

\begin{definition}\label{def: T}
Given vertices $V_i$, $V_j$, let $\vec{T}_{ij}$ be the set containing $V_i, V_j$ and all vertices that are distance $1$ in $\G_1$ from $V_i$ or $V_j$.
\end{definition}

\begin{lemma} \label{lem: tij has separating set}
If vertices $V_i, V_j$ are nonadjacent, the set $\vec{T}_{ij}$ contains a valid separating set $\vec{C} \subseteq \vec{T}_{ij}$ such that  $V_i \dsep V_j \given \vec{C}$.
\end{lemma}

Lemma~\ref{lem: tij has separating set} guarantees that the conditional probability distribution $\Pr(\vec{T}_{ij} \given u)$ has sufficient information to verify or falsify the adjacency of $V_i$ and $V_j$.

The rest of the construction of the $k$-MixProd instances is left to Appendix~\ref{apx: set up kmixprod}. Generally, it involves ensuring that the recovered $\G_1$ from Phase I is sparse enough to $d$-separate all $\vec{T}_{ij}$ from two other supervariables of sufficient cardinality. Lemma~\ref{lem: vertices needed for phase 2} tells us the number of vertices that are needed to set up an instance of $k$-MixProd. Algorithm~\ref{alg:FPCorrection} performs FP-edge correction with the statistics recovered from $k$-MixProd oracles. Theorem~\ref{lem: phase II runtime and vertex requirement} summarizes the results proved in Appendix~\ref{apx: set up kmixprod}.
\begin{theorem}\label{lem: phase II runtime and vertex requirement}
Phase II requires $(2\Delta + \Delta^3) \lceil \lg(k + 1) \rceil + 2(\Delta^2 + \Delta + 1) \lceil \lg(k) \rceil + 2 \Delta^2 + 2\Delta^3$ vertices and solves $k$-MixProd $\mathcal{O}(k \abs{\vec{E}} 2^{\Delta^2})$ times.
\end{theorem}
As with Lemma~\ref{lem: runtime phase1}, the $\Delta^2$ in the runtime of Theorem~\ref{lem: phase II runtime and vertex requirement} comes from the maximum size of a separating set and is therefore trivially bounded by $\abs{\vec{V}}$. A complete worked example walking through all phases on the graph of Figure~\ref{fig: example_setup} is given in Appendix~\ref{apx: worked example}.

\section{EMPIRICAL RESULTS} \label{sec: experiments}
Our theoretical results guarantee success with infinite data.  We now employ empirical tests for our hypothesis-based rank test and investigate the sensitivity of Phase I. Real-world datasets with a known ground truth are limited in this setting. As such, we utilize synthetic data on a broad class of functions and noise, the details of which are deferred to Appendix~\ref{apx: synth data details}. The synthetic approach is argued for in \citet{poinsot_position_2025}. All experiments are run using code provided at \href{https://github.com/honeybijan/causal_discovery_mixtures}{this GitHub link}. Use \texttt{pip install probrank} to install the rank test.

\subsection{Test 1: Rank Hypothesis Test}
This test compares the hypothesis test developed in Section~\ref{sec: hyp test} to a naive thresholding of singular values as in \citet{anandkumar2012learning}. We generate data from two graphs:
\begin{enumerate}
\vspace{-1em}
\item ``Connected'' $\G^{c}$: $V_1 \rightarrow V_2 \rightarrow V_3 \rightarrow V_4$
\item ``Split'' $\G^{s}$: $V_1 \rightarrow V_2$  \hspace{1cm} $V_3 \rightarrow V_4$
\vspace{-1em}
\end{enumerate}

\begin{figure}[h]
\centering
(a)\includegraphics[width=.35\textwidth]{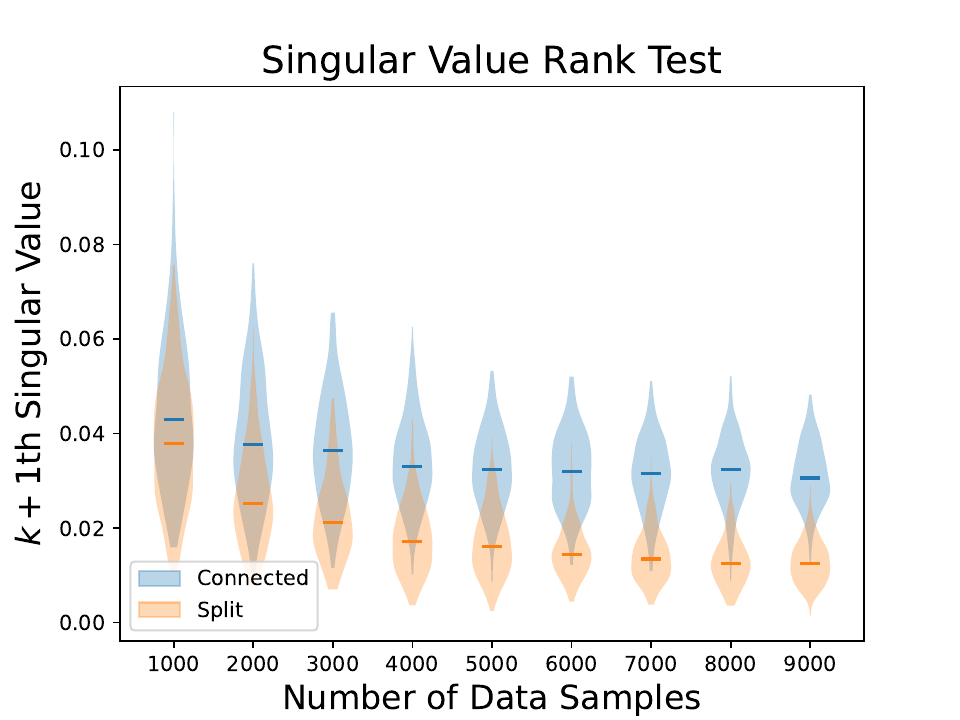}\\
(b) \includegraphics[width=.35\textwidth]{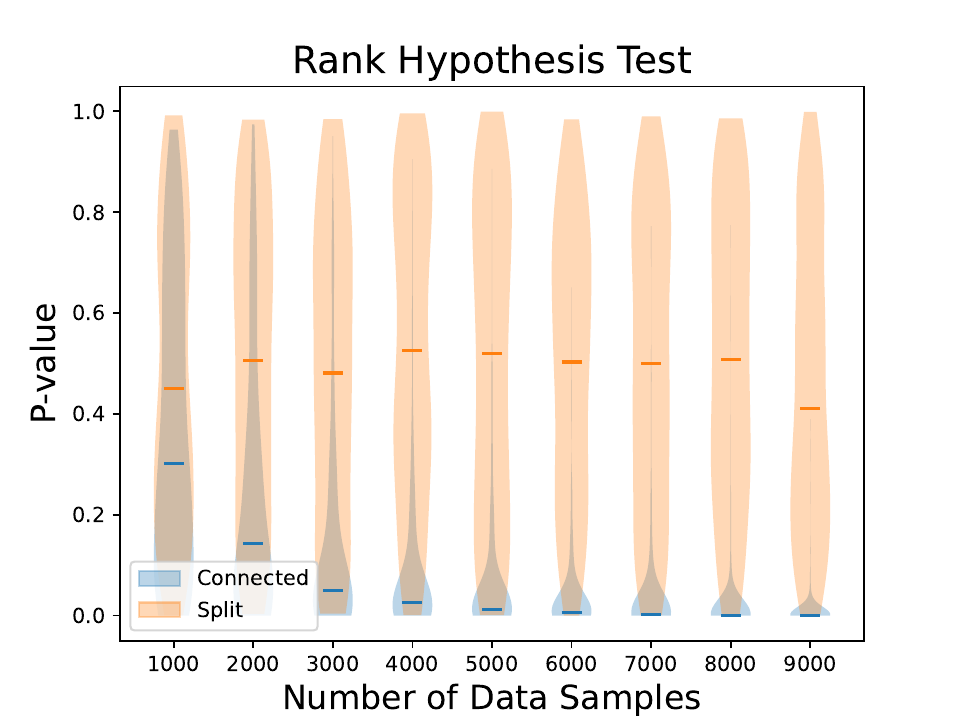}
\caption{The results of Test 1. We compare (a) the singular values and (b) the p-values of our hypothesis test for the connected (blue) and split (orange) models.} \label{fig:test1res}
\end{figure}

If we agglomerate our vertices into $\vec{S}^+_i = \{V_1, V_2\}$ and $\vec{S}^+_j = \{V_3, V_4\}$, then these two DAGs differ in that $\vec{S}^+_i \not \dsep^{\G^{c}} \vec{S}^+_j$ and $\vec{S}^+_i \dsep^{\G^{s}} \vec{S}^+_j$. Distinguishing between these two DAGs is notably difficult because $\vec{S}^+_i$ and $\vec{S}^+_j$ are only ``loosely'' connected by $V_2 \rightarrow V_3$.

We varied the number of samples from these distributions from $1000$ to $9000$ and studied the distributions of p-values and $k+1$th singular values across $200$ runs. The results are reported in Figure~\ref{fig:test1res}, showing that the hypothesis test has a significant difference in p-values beyond $2000$ samples (relative to the difference in the $k+1$th singular values). 
It is worth noting that it is almost impossible to choose a threshold for the $k+1$th singular value ahead of time, whereas hypothesis tests give a meaningful significance.

\subsection{Test 2: Varying Density}
\begin{figure}[h]
\centering
(a)\includegraphics[width=.35\textwidth]{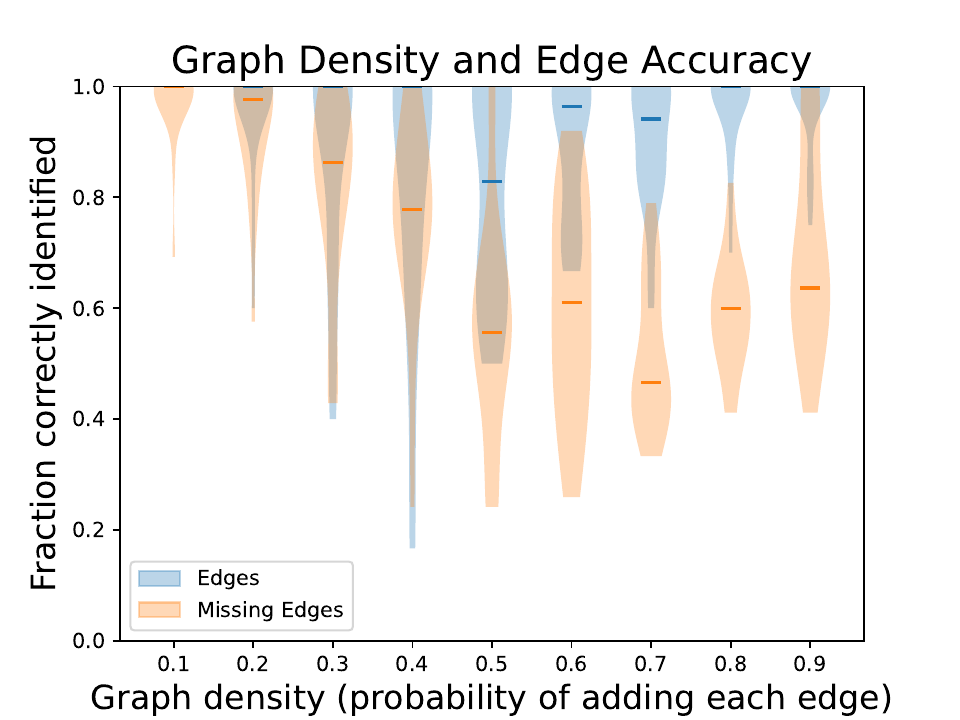}\\
(b)\includegraphics[width=.35\textwidth]{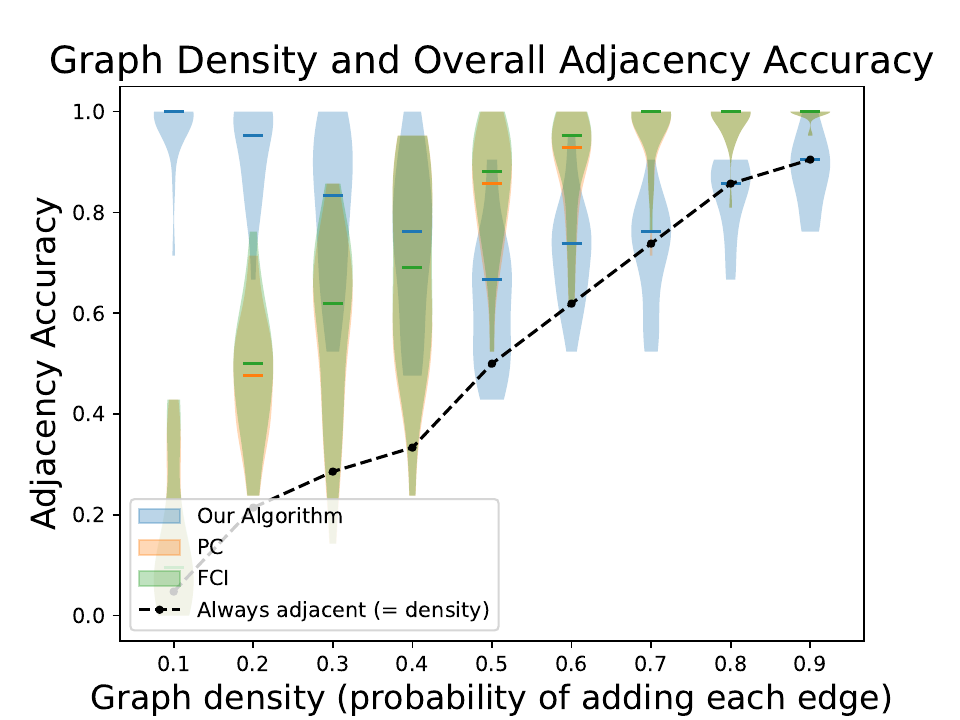}
\caption{The results of Test 2. In (a), blue gives the percentage of correctly identified edges, and orange is the percentage of correctly removed edges. In (b), our method is compared to PC and FCI; the dashed line is the trivial ``always-adjacent'' baseline (equal to the graph density), shown for reference. PC and FCI return graphs of roughly constant moderate density regardless of the truth---near-complete on sparse graphs---so they exceed the baseline substantially only once the ground truth is itself dense. Our method is strongest in the sparse regime covered by our theory.} \label{fig:test2res}
\end{figure}
In our second test, we explore the role of density in accurately detecting graph adjacency.  For this test, we sample random Erdös-Renyi undirected graph structures on $7$ vertices and orient them according to a random permutation of the vertices which determines the topological ordering. We vary the edge-probability from $.1$ to $.9$ in $.1$ increments,  sampling 20 graph structures for each.  Among these graphs, we draw $10,000$ datapoints and study the percentage of correctly recovered edges (significance threshold of $.7$ to account for many tests). Results are given in Figure~\ref{fig:test2res} (a).

Strong performance is shown at low-density graphs, which represent the regime of theoretical study in this paper. As the graph density increases, we fail to preserve edges (lower blue marks) and incorrectly report edges where none exist (lower orange marks). At high densities, our accuracy for detecting true edges returns to higher levels, but at the cost of occasionally adding false positive edges.
Figure~\ref{fig:test2res} (b) compares overall adjacency accuracy against PC and FCI, using the implementation of \cite{zheng2024causal}, with the trivial ``always-adjacent'' baseline (dashed, equal to the realized graph density) shown for reference. Adjacency accuracy must be read relative to this baseline, since on denser ground-truth graphs a larger fraction of pairs are genuinely adjacent and the baseline itself rises with density. Neither PC nor FCI models the latent global confounder $U$, and empirically both return graphs of roughly constant, moderate density---on the order of $14$ to $20$ of the $21$ possible edges---largely insensitive to the true sparsity. On sparse graphs this amounts to near-complete output (about $20$ edges when fewer than $2$ are real at density $0.1$), so their accuracy only marginally exceeds the baseline; their improvement on denser graphs is driven mainly by the coincidence of a fixed dense output with an increasingly dense ground truth, rather than by recovery of sparse structure. Our method is designed to remove exactly this confounder-induced dependence and attains the highest accuracy in the sparse regime, where PC and FCI are effectively at chance relative to the baseline. On denser graphs our accuracy is governed by a p-value threshold that trades removal of spurious edges against retention of true ones. Selecting this threshold is complicated by multiplicity: each candidate pair is subjected to many rank tests---over agglomerations, conditioning sets, and their assignments---so a fixed per-test level does not correspond to a fixed family-wise level. A principled choice, for instance, a correction that scales with the number of tests applied per pair, remains open.
\section{DISCUSSION}
This paper gives the first identifiability result and algorithm for causal DAG discovery under confounding by a mixture or latent class. It does so when the observables are Bernoulli, which contain only one parameter. This scarcity of parameters creates a unique identifiability problem, bringing about the need for ``agglomeration.'' In principle, our approach can be adapted for easier settings, e.g., discrete observables with higher cardinality alphabets or continuous observables with higher moments, \emph{without} the need to agglomerate (thereby making the problem significantly easier). Our result illustrates the relationship between the amount of observed information (given by the number of Bernoulli random variables and their graphical sparsity) and the amount of unobserved information (given by the number of latent classes $k$). Corollary~\ref{thm: ident} only supplies sufficient conditions, and qe leave the question of necesity as an open question.

The majority of our approach relies on the assumption of a known $k$. This restriction is unavoidable in all mixture settings (e.g., $k$-means). In practice, such a $k$ must be determined through trial and error. If $k$ is too large, then the resulting DAG will be very sparse. If $k$ is too small, then the resulting DAG will be very densely connected. Appendix~\ref{apx: additional experiments} repeats our experiments with (incorrectly) larger and smaller $k$, verifying that this property is very sensitive to the choice of $k$. This means that a practitioner could, in principle, tune $k$ until they receive results that are not close to cliques or empty graphs.

Fortunately, $k$ is sometimes known approximately. $k$ can also be estimated under certain circumstances, such as when some observables are \emph{known} to be independent (conditioned on $U$). In this case, a rank test can be used to determine $k$ --- i.e., the rank of the probability matrix between two conditionally independent variables will be rank $k$ under the previously discussed faithfulness assumptions.

We also assumed that the latent classes obeyed the \emph{same} DAG structure. This assumption can be relaxed: so long as there is enough shared sparsity from Phase I, we can detect statistical dependencies within \emph{at least one} graph structure. Phase II can then detect mild class-specific differences in graphical structure. This relaxation will, however, change the number of required vertices in nontrivial ways.

The results in this paper focus on discrete observables. Caution should be exercised when discretizing continuous results, as pointed out by \citet{sun_conditional_2025}. A more general approach could modify the probability matrix (Definition~\ref{def: prob mat}) to a matrix of moments, e.g. $\mat{M}[\mathbf{X}, \mathbf{Y}]_{ij} = E[X_i Y_j]$, utilizing kernel mean embedding \citep{muandet_kernel_2017} to compute conditional expectations.

Finally, it is worth noting that the hypothesis test of rank developed in Appendix~\ref{sec: hyp test} may have uses for causal representation learning (CRL) \citep{squires2023linear} and latent variable models \citep{silva2006learning, anandkumar2013learning, xie2020generalized, huang2022latent}. Additionally, the notion of agglomeration may have implications for the abstraction of coarse-grained concepts within CRL problems.

\begin{contributions}

    B.~Mazaheri conceived the idea, implemented the experiments, and wrote the paper.
    S. Gordon, Y. Rabani, and L. Schulman assisted with writing, editing, and verifying the results.
\end{contributions}

\begin{acknowledgements} 

    B. M. is partially supported by the Advanced Research Concepts (ARC) COMPASS program, sponsored by the Defense Advanced Research Projects Agency (DARPA) under agreement number HR001-25-3-0212.

    L. S. is partially supported by NSF grant CCF-2321079.
\end{acknowledgements}

\bibliography{refs}

\newpage

\onecolumn

\title{Causal Discovery in Mixtures of Populations\\(Supplementary Material)}
\maketitle

\appendix

\section{Worked Example}\label{apx: worked example}

We illustrate all three phases on the small graph $\G$ of Figure~\ref{fig: example_setup}, with a binary mixture source ($k = 2$) and Bernoulli observables $V_1, \ldots, V_6$. The observed edges of $\G$ are $V_4 \rightarrow V_1$, $V_1 \rightarrow V_2$, $V_2 \rightarrow V_6$, $V_4 \rightarrow V_5$, and $V_5 \rightarrow V_6$; the vertex $V_3$ has no observed neighbors. The mixture source $U$ points to every $V_i$, so in the observed marginal \emph{every} pair of variables is statistically dependent and no conditional independence holds. The maximum degree is $\Delta = 2$.

\paragraph{Why a naive approach fails.} Because $U$ confounds all pairs, a standard conditional-independence test never declares any pair independent, and the PC algorithm would return the complete graph. Our rank tests instead recover the structure by detecting, for suitable agglomerations, that the relevant probability matrix has nonnegative rank at most $k = 2$.

\paragraph{Phase I: agglomerated rank tests.} Since the observables are binary and $k = 2$, a single pair of variables yields only a $2 \times 2$ probability matrix, which cannot exceed rank $2$ and is therefore uninformative. We instead form supervariables of $\lceil \lg(k+1) \rceil = 2$ binary variables each, giving a $4 \times 4$ probability matrix.

Consider testing the non-adjacency of $V_1$ and $V_5$. In $\G$ they are d-separated by $\{V_4\}$: the fork $V_1 \leftarrow V_4 \rightarrow V_5$ is blocked by conditioning on $V_4$, and the path $V_1 \rightarrow V_2 \rightarrow V_6 \leftarrow V_5$ is blocked at the unconditioned collider $V_6$. We agglomerate $\vec{S}^+_i = \{V_1, V_2\}$ and $\vec{S}^+_j = \{V_5, V_3\}$ with conditioning set $\vec{C} = \{V_4\}$. One checks that $\vec{S}^+_i \indep^{\G}_d \vec{S}^+_j \given \vec{C}$: every path between the two groups passes through the collider $V_6$ or through $V_4$, both of which are blocked. Hence $(\vec{S}^+_i, \vec{S}^+_j)$ is an IPA. By Lemma~\ref{lem:rank_test}, for any assignment $V_4 = v$,
\[
\nnrank\!\big(\mat{M}[\{V_1,V_2\},\, \{V_5,V_3\} \mid V_4 = v]\big) \leq k = 2,
\]
because conditioning on $U$ as well makes the two groups independent, so the matrix decomposes into a sum of $k = 2$ rank-one terms (Equation~\ref{eq: decompose M}). Detecting this low rank lets Phase I delete all four candidate \emph{cross-group} edges at once---$V_1$--$V_5$, $V_1$--$V_3$, $V_2$--$V_5$, $V_2$--$V_3$---none of which is a true edge, while the within-group edge $V_1$--$V_2$ is untouched. Repeating over agglomerations and conditioning sets removes every non-edge that admits an IPA, leaving the skeleton $\{V_4\text{--}V_1,\, V_1\text{--}V_2,\, V_2\text{--}V_6,\, V_4\text{--}V_5,\, V_5\text{--}V_6\}$ with $V_3$ isolated.

\paragraph{A caveat the example makes concrete.} Our worst-case guarantee (Corollary~\ref{thm: ident}) does \emph{not} apply at this size: with $\Delta = 2,\, k = 2$ it would require $76$ vertices. Indeed, every vertex here is an ``early vertex'' (Definition~\ref{def: early vertices}), since each has fewer than $(2 + \Delta^2)\lceil \lg(k+1) \rceil - 2 = 10$ non-descendants, so Lemma~\ref{lem: all fp in H} gives no useful restriction on FP edges. Nevertheless, generically the rank tests still recover the correct skeleton, consistent with the experiments of Section~\ref{sec: experiments}. This is precisely the gap between the (loose) worst-case bound and typical behavior.

\paragraph{Phase II: $k$-MixProd correction (when needed).} Had Phase I left a false-positive edge between two early vertices $V_i, V_j$ (e.g., because too many candidate separating sets lay among the collider descendants of $V_i, V_j$), we would verify it as follows. Form $\vec{T}_{ij}$ from $V_i, V_j$ and their $\G_1$-neighbors (Definition~\ref{def: T}); find two further agglomerations $\vec{X}_1, \vec{X}_2$ at distance $> 2$ from $\vec{T}_{ij}$ and from each other; and condition on $\vec{Z}_{ij} = \Nb_2(\vec{X}_1) \cup \Nb_2(\vec{X}_2)$ so that $\vec{T}_{ij}, \vec{X}_1, \vec{X}_2$ become mutually independent given $U$ (Lemma~\ref{lem: d-sep 2 neighborhood}). A $k$-MixProd oracle applied to these three groups recovers $\Pr(\vec{T}_{ij} \mid u, \vec{z}_{ij})$. Because $\vec{Z}_{ij}$ may include collider descendants of $V_i, V_j$, we deconfound using Lemma~\ref{lem:remove conditioning} to obtain the unconditioned $\Pr(\vec{T}_{ij} \mid u)$, then run an ordinary conditional-independence test \emph{within} each class $u$. If $V_i \indep V_j \mid \vec{C}, u$ for all $u$, the edge is removed. For the graph above no FP edge arises, so Phase II is vacuous here; the construction on a larger graph is shown in Figure~\ref{fig:example_kmixprod}.

\paragraph{Phase III: orientation.} On the recovered skeleton, $V_2$ and $V_5$ are non-adjacent yet share the child $V_6$, forming the unshielded collider $V_2 \rightarrow V_6 \leftarrow V_5$, which we orient. The remaining edges $V_4\text{--}V_1$, $V_4\text{--}V_5$, and $V_1\text{--}V_2$ are forced neither by an immorality nor by a Meek rule, and so remain undirected within the CP-DAG (the isolated $V_3$ contributes nothing). The output is therefore the MEC of $\G$.

\section{D-separation} \label{apx: dsep}
We will rely on the concepts of \textbf{d-separation},  \textbf{active paths},  and \textbf{separating sets}.  Active paths are defined relative to sets of variables to be conditioned on (``conditioning sets'').  

A key concept in active paths is that of a collider, $C$, which takes the form $V_1 \rightarrow C \leftarrow V_2$ along an undirected path. Undirected paths through unconditioned variables with no colliders are considered active because they can ``carry dependence'' from one end of the path to the other. If a non-collider along one of these non-collider paths is conditioned on, then the path is inactive.

Collider steps of paths behave in the opposite fashion: paths with an unconditioned collider are inactive, whereas conditioning on a collider or any descendant of that collider ``opens'' the active path.

When two variables have an active path between them, we say that they are d-connected. If there are no active paths between the variables, we say they are d-separated.  A separating set between two variables is a set of variables which, when conditioned on, break all active paths between those variables. While only loosely described here, a full precise definition of d-separation can be found in \cite{pearl1988probabilistic} and \cite{pearl2009causality} (for a more extensive study).

\cite{pearl1988probabilistic} uses structural causal models to justify the \emph{local Markov condition}, which means that d-separation always implies independence and allows DAG structures to be factorized. It is possible that two d-connected variables by chance exhibit some unexpected \emph{statistical} independence.  This complication is often assumed away using ``faithfulness'' \citep{spirtes2000causation}, which ensures that d-connectedness implies statistical dependence.  Together, the local Markov condition and faithfulness give a correspondence statistical dependence and the graphical conditions of the causal DAG which can be leveraged for causal structure learning.

The following fact will be useful when arguing for the existance of separating sets.
\begin{lemma}[\cite{pearl2009causality}] \label{lem:parents d-sep}
If vertices $V_i, V_j$ are nonadjacent in $\G$,  either $\Pa^\G(V_i)$ or $\Pa^\G(V_j)$ are a valid separating set for $V_i, V_j$.
\end{lemma}

We will often want to bound the cardinality of separating sets relative to the degree bound of the graph ($\Delta$). When dealing with a separating set between two vertices $V_i, V_j$, Lemma~\ref{lem:parents d-sep} implies a simple upper bound of $\Delta$. Separating sets for \emph{sets} (or agglomerations) of vertices are significantly more complicated because conditioning may d-separate some pairs of vertices while d-connecting others.

To unify the treatment of separating sets, we will make use of \textbf{moral graphs}, which can be thought of as undirected equivalents of DAGs \citep{lauritzen1990independence}. 
We will denote the moral graph of $\G$ as $\G^{(m)}$. To transform $\G$ into $\G^{(m)}$, we add edges between all immoralities, i.e., nonadjacent vertices with a common child, sometimes called an unshielded collider. After this, we change all directed edges to undirected edges. 

A very useful fact from \cite{lauritzen1990independence} (also Eq. 1 in \cite{acid1996algorithm}) is that all separating sets $\vec{C} \subseteq \vec{V}$ for $\vec{S}, \vec{S}' \subseteq \vec{V}$ in $\G$ are also separating sets in the moralized subgraph on $\iAn(\vec{S} \cup \vec{S}' \cup \vec{C})$, denoted $(\G[\iAn(\vec{S} \cup \vec{S}' \cup \vec{C})])^{(m)}$.  This transforms complicated active path analysis into simple connectedness arguments on undirected moral graphs (of special subgraphs of $\G$). A convenient consequence of this transformation, which we will use throughout the paper, is Lemma~\ref{lem: delta2 sepset size bound}.
\begin{lemma}\label{lem: delta2 sepset size bound}
If $\G = (\vec{V}, \vec{E})$ has degree bound $\Delta$, then the size of the minimal separating set between $\vec{S}, \vec{S}' \subseteq \vec{V}$ is no larger than $\min(\abs{\vec{S}}, \abs{\vec{S}'}) \Delta^2$.
\end{lemma}
Lemma~\ref{lem: delta2 sepset size bound} is a consequence of the maximum increase in the degree of the moral graph. 
\begin{proof}
A key observation is that the moral graph of a subgraph of $\G$ has no additional edges relative to $\G^{(m)}$. That is, if $\G[\vec{W}] = (\vec{W}, \vec{F})$ is a subgraph of $\G = (\vec{V}, \vec{E})$ then the corresponding edge-sets of the moral graphs obey $\vec{F}^{(m)} \subseteq \vec{E}^{(m)}$ because adding vertices cannot have invalidated any previously contained immoralities. 

Abbreviate $(\G[\iAn(\vec{S} \cup \vec{S}' \cup \vec{C})])^{(m)}$ as $\G_{\vec{C}}^{(m)}$. Even though we do not know what $\vec{C}$ is, we know that the $1$-neighborhood of $\vec{S}$ in $\G_{\vec{C}}^{(m)}$ suffices as a separating set.  $\G^{(m)}$ has all of the edges of $\G_{\vec{C}}^{(m)}$, so
\begin{equation}
\Nb_1^{\G_{\vec C}^{(m)}}(\vec{S}) \subseteq \Nb_1^{\G^{(m)}}(\vec{S}).
\end{equation}
Note that $\Nb_1^{\G^{(m)}}(\vec{S})$ is not necessarily a separating set for $\vec{S}, \vec{S}'$ in $\G^{(m)}$, in fact it may include some vertices in $\vec{S}$ itself. However, the size of the separating set is bounded by $\abs{\Nb_1^{\G^{(m)}}(\vec{S})}$, which is no larger than $\abs{\vec{S}} \Delta^2$. As we chose $\vec{S}$ arbitrarily, this bound also holds for $\vec{S}'$.
\end{proof}

Another helpful result from moral graphs is Lemma~\ref{lem: moral graph dsep}, which will help us when we prove the existence of separating sets.

\begin{lemma}[Corollary of Theorem 1 in \cite{acid1996algorithm}] \label{lem: moral graph dsep}
For DAG $\G = (\vec{V}, \vec{E})$, and $\vec{S}, \vec{S}' \subseteq \vec{V}$,  separating sets $\vec{S}, \vec{S}$ in $(\G[\iAn(\vec{S} \cup \vec{S}')])^{(m)}$ are also separating sets in $\G$.
\end{lemma}

\section{Identifying Mixtures of Discrete Products}
\label{apx: kmixprod}
The main tool used in \cite{gordon2023causal} is a solution to identifying discrete $k$-mixtures of product distributions (\textbf{$k$-MixProd}) - i.e. $\vec{X} = X_1, \ldots, X_n$ and latent global confounder or ``source'' $U$ such that $X_i \indep X_j \given U$ for all $i, j$.  The key complexity parameter for identifiability is $k$, the cardinality of the support of $U$.\footnote{We will refer to the cardinality of the support of discrete random variables as their cardinality.}

At its core, $k$-MixProd shows how coincidences of multiple independent events reveal information about their confoundedness.  Of course,  it is possible for $U$ with sufficiently large $k$ to completely control the distribution on $\vec{X}$. For example,  for binary $X_i \in \{0, 1\}$,  a cardinality of $k = 2^n$ would be sufficient to assign each binary sequence in $\vec{X}$ to a latent class in $U$.  Such a powerful $U$ could generate \emph{any} desired probability distribution on $\vec{X}$ by simply controlling the probability distribution on $U$.  Limiting $k$, however, limits the space of marginal probability distributions on $\vec{X}$, eventually giving rise to identifiability.

Under a cardinality bound $k$ on the support of $U$, \cite{allman2009identifiability} showed that $n \geq \Omega(\log(k))$ is sufficient for the generic identification of $k$-MixProd.  In other words,  other than a Lebesgue measure $0$ set of exceptions, most instances of $k$-MixProd have a one-to-one correspondence with their observed statistics (the probability distribution on $\vec{X}$ marginalized over $U$) and generating model (up to a set of $k!$  models with permuted labels of $U$). 

\cite{tahmasebi2018identifiability} demonstrated that a linear lower bound ($n \geq 2k-1$), in conjunction with a separation condition in the distributions of $X_i \given U$, is sufficient to guarantee identifiability.  The best-known algorithm for identification is given in \cite{gordon2024identification}, which nearly matches the known lower bounds for sample complexity. This paper will use the result from \cite{allman2009identifiability}, but our methods easily extend to stronger identifiability conditions with modifications in the sparsity requirements.

\section{HYPOTHESIS TEST FOR RANK}\label{sec: hyp test}
We build a test for the null-hypothesis that $\mat{A} \coloneqq \mat{M}[V, V' \given \vec{C}]$ has rank $\leq k$. For this derivation, we use $V_i = V$ and $V_j=V'$ to free up $i,j$ for indexing entries of matrices. At a high level, the test asks whether the smallest singular values of $\mat{\hat{A}}$ (those beyond the first $k$) are jointly indistinguishable from zero given sampling noise; if so, $\mat{A}$ has rank $\leq k$. In the notation that follows, $\mat{B}$ denotes the (rescaled) estimation error and $\varepsilon(N)$ its scale, $\mat{U}_2$ and $\mat{V}_2$ collect the left/right singular vectors associated with the ``extra'' singular values beyond rank $k$, $\mat{L}$ is the diagonal matrix of those extra singular values (vectorized as $\hat{l}$), and $f$ is the degrees of freedom of the limiting $\chi^2$ distribution. This section will closely follow \citet{ratsimalahelo2001rank}'s test of the rank of a matrix $\mat{A}$ under an asymptotically normal perturbation
\begin{equation}
    \mat{\hat{A}} = \mat{A} + \varepsilon(N) \mat{B}
\end{equation} 
with $\varepsilon(N) \propto 1/\sqrt{N}$ for $N$ samples.

To apply the test of rank to our setting, we break down $\mat{\hat{A}}$ into the sum of indicator functions which are applied to each of $N$ data entries.

\begin{align}
    \mat{C}(v, v')_{i,j} &\coloneqq \frac{1}{N} \Indic{v=i}\Indic{v'=j},\\
    \mat{\hat{A}}_{i, j} &= \sum_{v, v' \in \text{Data}} \mat{C}(v, v')
\end{align}

From this decomposition, $\varepsilon(N) \mat{B} = \mat{A}_{i, j} - \mat{\hat{A}}_{i, j}$ approaches normal with variance proportional to $1/N$ for larger $N$. This means that the covariance of $\mat{B}$ can be approximated using the covariance matrix for the entries of $\mat{C}$, which are Bernoulli random variables with probability estimates given by the entries of $\mat{\hat{A}}$.

\begin{align}\label{eq: bernouli var}
 \Cov(\mat{C}_{ij}, \mat{C}_{i'j'})= \begin{cases} \mat{A}_{ij} (1 - \mat{A}_{ij}) & \text{if }i=i', j=j'\\
 -\mat{A}_{ij}\mat{A}_{i'j'} &\text{otherwise.}
\end{cases}
\end{align}

First, take the SVD of the empirical $m \times n$ matrix, $\mat{\hat{A}} = \mat{U}\mat{D}\mat{V}^\top$, with the singular values ordered decreasingly along the diagonal of $\mat{D}$. Let $\mat{U}_2$ (columns $k+1,\dots,m$ of $\mat{U}$) and $\mat{V}_2$ (columns $k+1,\dots,n$ of $\mat{V}$) collect the \emph{trailing} left and right singular vectors---those associated with the ``extra'' singular values beyond the first $k$. Let $\mat{L} \coloneqq \mat{U}_2^\top \mat{\hat{A}} \mat{V}_2$ be the (diagonal) matrix of those extra singular values, and vectorize it to $\hat{l}$ by stacking its columns.

\citet{ratsimalahelo2001rank}'s estimator uses the covariance matrix of the entries of $\mat{\hat{B}} - \mat{B}$, denoted $\Sigma$. For our setting, we substitute the covariance matrix for the entries of $\mat{C}$ whose empirical estimates are given by Equation~\ref{eq: bernouli var}.
\begin{equation}
\Sigma_{i + jm,  i' + j'm} \coloneqq \Cov(\mat{C}_{ij}, \mat{C}_{i'j'}).
\end{equation}

Now, with $(\cdot)^\dag$ indicating the Moore-Penrose pseudoinverse and $\otimes$ indicating the Kronecker product, define the projected covariance
\begin{equation}
\mat{\hat{Q}} \coloneqq (\mat{V}^\top_2 \otimes \mat{U}^\top_2)\, \hat{\Sigma}\, (\mat{V}_2 \otimes \mat{U}_2).
\end{equation}
Following \citet{ratsimalahelo2001rank}, the statistic $N \hat{l}^\top \mat{\hat{Q}}^{\dag}\hat{l}$ converges to $\chi^2_{f}$ under $N$ samples, where the degrees of freedom $f = (m - k)(n - k)$ count the tested ``extra'' singular directions (for the square $m \times m$ case, $f = (m-k)^2$). Equivalently, $f$ is the rank of $\mat{\hat{Q}}$. Note that the pseudoinverse is taken \emph{after} projecting the covariance onto the extra singular directions, not before. This gives us a hypothesis test that is specifically designed for the rank test from this section. This approach is tested and compared to thresholding the $k+1$th singular value in Section~\ref{sec: experiments}.

\section{Further Runtime Discussion}
\label{apx: further runtime discussion}
The non-negative rank of a $k+1 \by k+1$ matrix (as used for our algorithm) can be solved in time $k^{\mathcal{O}(k^2)}$ \citep{moitra2016almost}. In the absence of non-negative rank tests,  \cite{anandkumar2012learning} demonstrated that regular rank tests generally work well in place of non-negative rank tests in practice.  We will develop a hypothesis test for matrix-rank that requires $\mathcal{O}(k^6)$ operations to invert a $(k+1)^2 \times (k+1)^2$ covariance matrix of the estimated elements of our $k+1 \by k+1$ matrix.

Our solution also requires solving  $k$-MixProd on 3 variables of cardinality $\mathcal{O}(k)$, which corresponds to decomposing a $\mathcal{O}(k) \times \mathcal{O}(k) \times \mathcal{O}(k)$ tensor into rank $1$ components. When the rank of such a tensor is known to be linear in $k$,  this decomposition can be solved in $O(k^{6.05})$ \citep{ding2022fast}, though the problem generally suffers from instability, with a worst-case sample complexity that is exponential in $k$ \citep{RSS14}.  This step may be considered optional, as it is used to refine a small number of ``false-positive'' adjacencies confined to a provably small subset of the DAG.

\section{Utilizing $\G_1$ to set up $k$-MixProd}\label{apx: set up kmixprod}

The first step to recovering $\Pr(\vec{T}_{ij} \given u)$ is to select some $\vec{Z}_{ij}$ and recover $\Pr(\vec{T}_{ij} \given u, \vec{z}_{ij})$ using instances of $k$-MixProd induced on the conditional probability distribution $\Pr(\vec{V} \given \vec{z}_{ij})$.  Recall that $k$-MixProd requires three independent variables of sufficient cardinality.  Hence,  we must find $\vec{X}_1, \vec{X}_2,  \vec{T}_{ij}$ which are sufficiently large, and d-separated from each other by $\vec{Z}_{ij}$ in $\G$. See Figure~\ref{fig:example_kmixprod} for an example.

\begin{figure}[h]
    \centering
    \scalebox{.7}{
    \begin {tikzpicture}[-latex ,auto ,node distance =1.5 cm and 1.5 cm ,on grid , very thick, state/.style ={ circle, draw, minimum width =.5 cm}, cstate/.style ={ circle, draw, minimum width =.5 cm, fill=black, text=white}]
    		\filldraw[ultra thick, color = red, fill opacity = .1, text opacity = 1] (5.05,-4.05) -- (0.8,-4.05) -- (0.8,-2.6) -- (-2.1,-2.6)-- (-2.1,-1.45)-- (5.05,-1.45) --cycle
            node[below left] {collider descendants};
            \filldraw[color=blue, fill opacity = .1, very thick, text opacity = 1](-5.1,-4.1) rectangle (-.8,-1.3 ) node[above] {$\vec{T}_{ij}$};
            \filldraw[color=blue, fill opacity = .1, very thick, text opacity = 1](.9,-4.1) rectangle (2.1,-1.3 )
            node[above] {$\vec{X}_1$};
             \filldraw[color=blue, fill opacity= .1, very thick, text opacity = 1](3.9,-4.1) rectangle (5.1,-1.3 )
            node[above] {$\vec{X}_2$};
            \node[state, dashed, blue] (U) {$U$};
            \node[state] (V5) [below left = 2cm and 1.5cm of U] {$V_5$};
            \node[state] (V6) [below = of V5] {$V_6$};
            \node[state] (V3) [left  = of V5] {$V_3$};
            \node[state] (V4) [below = of V3] {$V_4$};
            \node[state] (V1) [left = of V3] {$V_1$};
            \node[state] (V2) [below = of V1] {$V_2$};
            \node[cstate] (V7) [right = of V5] {$V_7$};
            \node[cstate] (V8) [below = of V7] {$V_8$};
            \node[state] (V9) [right = of V7] {$V_9$};
            \node[state] (V10) [below = of V9] {$V_{10}$};
            \node[cstate] (V11) [right = of V9] {$V_{11}$};
            \node[cstate] (V12) [below = of V11] {$V_{12}$};
            \node[state] (V13) [right = of V11] {$V_{13}$};
            \node[state] (V14) [below = of V13] {$V_{14}$};
            \path[dotted, red, line width = 1mm] (V3) edge[-] (V4);
            \path (V1) edge (V3) (V1) edge (V4) (V2) edge (V4) (V3) edge (V5) (V4) edge (V5) (V6) edge (V8);
            \path (V5) edge (V7) (V7) edge (V10) (V7) edge (V9) (V8) edge (V9) (V10) edge (V12) (V9) edge (V12) (V9) edge (V12) (V11) edge (V13) (V13) edge (V14) (V12) edge (V14) (V9) edge (V11);
            \path[blue] (U) edge[bend right=20] (-2.95, -1.3);
            \path[blue] (U) edge[bend left=20] (1.45, -1.3);
            \path[blue] (U) edge[bend left=20] (4.45 ,-1.3);
    \end{tikzpicture}
    }
    \caption{The given graph has an FP edge between $V_3$ and $V_4$, indicated by a dashed line. This FP edge is induced by a large set of collider descendants, shown in red.  Conditioning on $\vec{Z}_{ij} = \{V_7, V_8, V_{11}, V_{12}\}$ creates an instance of $k$-MixProd on $\vec{T}_{ij}, \vec{X}_1, \vec{X}_2$.  Notice that $\vec{Z}_{ij}$ are all collider descendants of $V_3, V_4$, which means that the conditional distribution, $\Pr(\vec{T}_{ij} \given \vec{z}_{ij}, u)$, recovered by $k$-MixProd will not be sufficient for detecting the FP edge. This obstacle will be solved in Subsection~\ref{sec:final_adjustment}.} \label{fig:example_kmixprod} 
 \end{figure}

Of course, we have access to $\G_1$, not $\G$.  $\G_1$ contains no orientations\footnote{It is, in principle,  possible to orient immoralities within $\G_1$ at this stage, but this gives no complexity improvements.} and may contain extra false-positive edges. We will need to build a conditioning set $\vec{Z}_{ij}$ that guarantees an instance of $k$-MixProd. While Markov boundaries cannot be computed without $\G$,  we can easily use $\G_1$ to find a superset that \emph{contains} the Markov boundary of a given vertex.
\begin{lemma}\label{lem: 2 neighborhood has MB}
The $2$-neighborhood of $\vec{X} \subseteq \vec{V}$ in $\G_1$ contains $\Mb^\G(\vec{X})$.
\end{lemma}
\begin{proof}
The distance between $X\in \vec{X}$, and $V \in \vec{V}$ in $\G_1$ is less than or equal to the distance in $\G$, because $\vec{E}_1 \supseteq \vec{E}$. This means that the $2$-neighborhood of $\vec{X}$ in $\G_1$ includes the $2$-neighborhood of $\vec{X}$ in $\G$. Furthermore because all vertices in  $\Mb^\G(\vec{X})$ are distance $\leq 2$ from at least one $X \in \vec{X}$, we have that $\Mb^\G(\vec{X})$ is contained in the $2$-neighborhood of $\vec{X}$ in $\G_1$.
\end{proof}
Recall that the definition of a Markov boundary of $\vec{X}$ is a set such that all supersets of the Markov boundary d-separate every other variable from $\vec{X}$. As such, three sets can be d-separated by conditioning on the Markov boundaries of just two of those sets. The following lemma makes this statement precise for $2$-neighborhoods.
\begin{lemma} \label{lem: d-sep 2 neighborhood}
    Suppose we have three disjoint sets $\vec{X}_0, \vec{X}_1, \vec{X}_2$ such that for all $X_i \in \vec{X}_i, X_j \in \vec{X}_j$ ($i \neq j \in \{0, 1, 2\}$), $X_i$ and $X_j$ are distance $>2$. For an arbitrary (disjoint) choice of $\vec{Z}, \vec{Z}' \in \{\Nb_2(\vec{X}_0), \Nb_2(\vec{X}_1), \Nb_2(\vec{X}_2)\}$, then $\vec{X}_0, \vec{X}_1, \vec{X}_2$ are mutually independent given $\vec{Z} \cup \vec{Z}'$.
\end{lemma}
\begin{proof}
    Because both sets are distance $>2$ apart, the 2-neighborhoods do not overlap with either set. By Lemma~\ref{lem: 2 neighborhood has MB}, the conditioning set contains the Markov boundaries for both sets, so they are d-separated.
\end{proof}

A worst-case analysis leads to the following minimum number of vertices for identifiability.

\begin{lemma} \label{lem: vertices needed for phase 2}
If $\abs{\vec{V}} \geq (2\Delta + \Delta^3) \lceil \lg(k + 1) \rceil + 2(\Delta^2 + \Delta + 1) \lceil \lg(k) \rceil + 2 \Delta^2 + 2\Delta^3$, then there exists $\vec{X}_1, \vec{X}_2$ each with cardinality $\lceil \lg(k) \rceil$ such that $\vec{X}_1, \vec{X}_2, \vec{T}_{ij}$ are mutually independent given $\Nb_2(\vec{X}_1) \cup \Nb_2(\vec{X}_2)$.
\end{lemma}
\begin{proof}
The proof will involve multiple applications of the fact that a degree bound $\Delta$ allows us to bound $\abs{\Nb_1(\vec{X})} \leq \Delta \abs{\vec{X}}$. We will also use $\Nb_2(\vec{X}) = \Nb_1(\Nb_1(\vec{X})) \cup \Nb_1(\vec{X})$.

The added complexity comes from the fact that our degree bound $\Delta$ applies to $\G$, not $\G_1$, which may have extra FP edges within the early vertices $\vec{H}$. The worst-case scenario involves an $\vec{H}$ that is made a clique in $\G_1$, with all possible FP edges present. As such, one can count the maximum number of vertices needed to find a $\vec{X}_1$ and $\vec{X}_2$ with the distance constraints needed for Lemma~\ref{lem: d-sep 2 neighborhood} (separation by distance $2$).

In order for $\vec{X}_1$ and $\vec{X}_2$ to be distance $2$ from $\vec{T}_ij$, $\Nb_2(\vec{T}_{ij})$ must be disjoint from $\vec{X}_1$ and $\vec{X}_2$. $\Nb_2(\vec{T}_{ij})$'s cardinality can be bounded by first decomposing the $1$-neighborhood into vertices in $\vec{H}$ and vertices not in $\vec{H}$.
\begin{equation}
    \Nb_1(\vec{T}_{ij}) = \left(\Nb_1(\vec{T}_{ij}) \cap \vec{H} \right) \cup \left(\Nb_1(\vec{T}_{ij}) \setminus \vec{H} \right).
\end{equation}
Notice that the first term is the full set $\vec{H}$ if $\vec{H}$ is a clique, since $\vec{T}_{ij} \subseteq \vec{H}$. The cardinality of the second term can be bounded using $\Delta$, because it does not involve any potential for FP edges. 
\begin{equation}
    \Nb_1(\vec{T}_{ij}) \leq \abs{\vec{H}} + \Delta \abs{\vec{T}_{ij}}.
\end{equation}
When extending this neighborhood to the $2$-neighborhood, we have already exhausted all of $\vec{H}$, meaning that the degree bound of $\Delta$ still applies. The degree of the vertices in $\vec{H}$ is now $\Delta + 1$, meaning that those vertices have $(\Delta - 1) \abs{\vec{H}}$ distance 2 neighbors, plus no more than $\vec{H}$ distance 1 neighbors. Similarly, $\vec{T}_{ij}$ has no more than $\Delta^2\vec{T}_{ij}$ distance 2 neighbors and no more than $\Delta \vec{T}_{ij}$ distance 1 neighbors. This gives,
\begin{equation}
    \Nb_2(\vec{T}_{ij}) + \abs{\vec{T}_{ij}} \leq (\Delta + 1 - 1) \abs{\vec{H}} + (\Delta^2 + \Delta + 1) \abs{\vec{T}_{ij}}.
\end{equation}
Lemma~\ref{obs: size of h and degree} bounds the cardinality of $\vec{H}$. The cardinality of $\vec{T}_{ij}$ is no more than two vertices for the $V_i, V_j$ in question, and their distance one neighborhood, which includes connections to no more than $\Delta-1$ additional vertices (see Definition~\ref{def: T}), which means $\abs{\vec{T}_{ij}} \leq 2 \Delta$. Combining these gives,
\begin{align*}
    \Nb_2(\vec{T}_{ij}) &\leq \Delta ((2 + \Delta^2) \lceil \lg(k + 1) \rceil -2) + (\Delta^2 + \Delta + 1) \cdot 2\Delta = (2\Delta + \Delta^3) \lceil \lg(k + 1) \rceil + 2 \Delta^2 + 2\Delta^3.
\end{align*}
We need no more than an additional $(\Delta^2 + \Delta + 1) \lceil \lg(k) \rceil$ vertices for each $\abs{\vec{X}_1} = \abs{\vec{X}_2} = \lceil \lg(k) \rceil$ and their $2$-neighborhoods. Summing all these together gives the total number of vertices needed.
\end{proof}

Having now shown that two $\vec{X}_1, \vec{X}_2$ of cardinality $\lceil \lg(k) \rceil$ may be isolated, we note that these two variables can be agglomerated into super-variables $\vec{S}_1, \vec{S}_2$ with cardinality at least $k$ each. Combined with $\vec{T}_{ij}$, which may be agglomerated into a variable $\vec{S}_3$ with at least cardinality $2$ (trivially), all three $\vec{S}_1, \vec{S}_2, \vec{S}_3$ are conditionally independent and meet the requirements of Lemma~\ref{lemma: kruskal} for generic identification of $\Pr(\vec{T}_{ij} \given U)$. This proves the vertex requirement portion of Corollary~\ref{thm: ident} and Theorem~\ref{lem: phase II runtime and vertex requirement}.

\subsection{Aligning multiple $k$-MixProd runs}
$k$-MixProd distributions are symmetric with respect to the $k!$ permutations on the label of their source. For this reason, there is no guarantee that multiple calls to a $k$-MixProd solver will return the same permutation of source labels.  

To solve this, \cite{gordon2023causal} noticed that any two solutions to $k$-MixProd problems that share the same conditional probability distribution for at least one ``alignment variable'' can be ``aligned'' by permuting the source labels until the distributions on that variable match up.  We will only need alignment along runs for different assignments to each $\vec{Z}_{ij}$,  used in the next section.  Explicitly,  two assignments $\vec{z}_{ij}$ and $\vec{z}'_{ij}$, need least one $\vec{X}^* \in \{\vec{T}_{ij}, \vec{X}_1, \vec{X}_2\}$ such that $\mb_{\vec{z}_{ij}}(\vec{X}^*)$ and $\mb_{\vec{z}'_{ij}}(\vec{X}^*)$ are the same, in order for alignability to be satisfied.

To align sets of $k$-MixProd results which are not all pairwise alignable, \cite{gordon2023causal} introduced the concept of an ``alignable set of runs'' for which chains of alignable pairs allow alignability.

\begin{lemma} \label{lem: alignability}
The set of $k$-MixProd instances on the same $\vec{T}_{ij}, \vec{X}_1, \vec{X}_2$ with all possible assignments $\vec{z}_{ij}$ to $\vec{Z}_{ij}$ is alignable.
\end{lemma}
\begin{proof}
Any two runs with assignments $\vec{z}_{ij}$ and $\vec{z}'_{ij}$ that differ in their assignment to only one variable are alignable. Therefore, any two non-alignable runs can be aligned using a chain of Hamming-distance one alignments.
\end{proof}

\subsection{Recovering the unconditioned within-source distribution} \label{sec:final_adjustment}
After all our calls to the $k$-MixProd oracle, we have access to $\Pr(\vec{T}_{ij} \given u, \vec{z}_{ij})$ and $\Pr(u \given \vec{z}_{ij})$ for every assignment $\vec{z}_{ij}$ and $u$.
$\Pr(\vec{T}_{ij} \given u, \vec{z}_{ij})$ is insufficient to determine the adjacency of $V_i, V_j$ because $\vec{Z}_{ij}$ may contain vertices in the collider descendants of $V_i, V_j$,  prohibiting the discovery of a separating set within $\vec{T}_{ij}$.

Instead, we must recover $\Pr(\vec{T}_{ij} \given u)$, which is not conditioned on $\vec{Z}_{ij}$. To do this, we can apply the law of total probability over all possible assignments to $\vec{Z}_{ij}$.

\begin{equation}
\Pr(\vec{T}_{ij} \given u) = \sum_{\vec{z}_{ij}} \Pr(\vec{z}_{ij} \given u) \Pr(\vec{T}_{ij} \given \vec{z}_{ij}, u) 
\end{equation}

We can obtain $\Pr(\vec{z}_{ij} \given u)$ by using Bayes rule on the $k$-MixProd output, $\Pr(u \given \vec{z}_{ij})$.

\begin{equation}
\Pr(\vec{z}_{ij} \given u) = \frac{\Pr(u \given \vec{z}_{ij}) \Pr(\vec{z}_{ij})}{\Pr(u)}.
\end{equation}

 $\Pr(\vec{z}_{ij})$ can be obtained by counting the frequency of $\vec{z}_{ij}$ in the data. In addition,  $\Pr(u) = \sum_{\vec{z}_{ij}} \Pr(\vec{z}_{ij}) \Pr(u \given \vec{z}_{ij})$ is computable by the law of total probability after the runs for each assignment $\vec{z}_{ij}$, have been aligned.  Alternatively,  we can observe that all $\Pr(\vec{z}_{ij} \given u)$ have the same $\Pr(u)$ in their denominator, so we can scale these values so that  $\sum_{\vec{z}_{ij}}\Pr(\vec{z}_{ij} \given u) = 1$. 

\begin{lemma} \label{lem:remove conditioning}
We can compute $\Pr(\vec{T}_{ij} \given u)$ using known quantities, 
\begin{equation*}
 \Pr(\vec{T}_{ij} \given u)  = \frac{\sum_{\vec{z}_{ij}} \Pr(u \given \vec{z}_{ij}) \Pr(\vec{z}_{ij}) \Pr(\vec{T}_{ij} \given \vec{z}_{ij}, u)}{ \sum_{\vec{z}_{ij}} \Pr(\vec{z}_{ij}) \Pr(u \given \vec{z}_{ij})}.
\end{equation*}
\end{lemma}
$ \Pr(\vec{T}_{ij} \given u)$ is a completely deconfounded distribution on which we can run the PC-algorithm. The full procedure is given in Algorithm~\ref{alg:FPCorrection}, in which we use breadth first search (BFS) to find $\vec{T}_{ij}, \vec{X}_1, \vec{X}_2$ such that they are separated by distance $2$ in $\G'$, followed by alignment and Lemma~\ref{lem:remove conditioning} in order to remove all of the false-positive edges from $\G_1$.
\begin{algorithm} 
    \caption{Phase II: Detection and correction of FP edges.}\label{alg:FPCorrection}
    \DontPrintSemicolon
    \KwIn{$\Pr(\vec{V})$ marginalized over $U$, a black box solver for $k$-MixProd, and $\G_1 = (\vec{V}, \vec{E}_1)$ from the output of Algorithm~\ref{alg:phase1}.}
 	\KwOut{$\G_2 = (\vec{V}, \vec{E}_2)$, an undirected skeleton of $\G$ and separating sets for nonadjacencies (vertices not in $\vec{E}_2$) .}
 	
		Start with $\vec{E}_2 \gets \vec{E}_1$. 		
 		
 		\For{each $\{V_i, V_j\} \in \vec{E}_1$}{
       	
       		Retrieve $\vec{T}_{ij},  \vec{X}_1, \vec{X}_2, \vec{Z}_{ij}$ using BFS.
       	
       		\For{each assignment $\vec{z}_{ij}$}{
       		
       			Run the $k$-MixProd solver on $\vec{T}_{ij}, \vec{X}_1, \vec{X}_2$ on $\Pr(\vec{V} \given \vec{z}_{ij})$. 
       		}
       	
       		Perform alignment of the $2^{\vec{z}_{ij}}$ runs to retrieve $\Pr(\vec{T}_{ij} \given \vec{Z}_{ij}, U)$.
       	
       	 Calculate $\Pr(\vec{T}_{ij} \given u)$ for every $u$ using Lemma~\ref{lem:remove conditioning}.
       	 
       	 Run PC or any other structure learning algorithm on $\Pr(\vec{T}_{ij} \given u)$ to find a separating set $\vec{C}_{ij}$ (or verify adjacency) for $V_i, V_j$.  If $V_i \indep V_j \given \vec{C}_{ij}, u$ for all $u$, remove $\{V_i, V_j\}$ from $\vec{E}_2$ and store $\vec{C}_{ij}$.
       	 }
\end{algorithm}

The runtime of Lemma~\ref{alg:FPCorrection} is dominated by the number of $k$-MixProd runs, since BFS is $O(\abs{\vec{V}} + \abs{\vec{E}_1})$, and the number of runs of $k$-MixProd are also linearly proportional to $\abs{\vec{E}_1}$ while also incurring the time costs of running $k$-MixProd.
\begin{lemma}\label{lem:phase2 runtime}
Algorithm~\ref{alg:FPCorrection} requires solving $k$-MixProd $\mathcal{O}((\Delta^2 \lg(k)) + \abs{\vec{E}})k 2^{\Delta^2})$ times.
\end{lemma}
\begin{proof}
This algorithm requires running $k$-MixProd for every possible assignment to the conditioning set $\vec{z}_{ij}$, for which we have $\abs{\vec{Z}_{ij}} \leq (\lg(k) + 2) \Delta^2$ total binary variables.  This gives an upper bound of $2k 2^{\Delta^2}$ runs of $k$-MixProd for each edge in $\G_1$, of which there are no more than $\abs{\vec{E}} + \abs{\vec{H}}$ which is $\mathcal{O}(\Delta^2 \lg(k)) + \abs{\vec{E}})$.
\end{proof}

\section{Synthetic Data Details} \label{apx: synth data details}
\subsection{Structural Equation Setup}
SCMs are made up of a graphical structure and accompanying structural equations.  We focus our tests primarily on varying the graphical structure, using a standard set of structural equations on these graphs. Our $U$ are generated using a fair coin ($k=2$), and all other vertices are Bernoulli random variables with bias $p_V$ determined by $V$'s parents (including $U$):
\begin{equation}
p_V =\frac{1 + \sum_{W \in \Pa^{\G'}(V)} W}{\abs{\Pa^{\G'}(V)} + 2}.
\end{equation}
Structural equations of this form have a reasonable strength between vertices that is decreased relative to the in-degree. Since vertices can be permuted, this definition describes the broad class of distributions with sufficient separation.

\subsection{Non-Symmetric Structural Equations}
The structural equation used above is symmetric: it depends on the parents of $V$ only through their sum, not through which parents are active. To probe the nonparametric generality of our method, we additionally consider a non-symmetric family in which each parent is assigned an independent random weight. Concretely, for each vertex $V$ we draw weights $w_{VW} \sim \mathrm{Unif}[.5,1]$ for each $W \in \Pa^{\G'}(V)$ and a bias $b_V \sim \mathrm{Unif}[-1,1]$, and set
\begin{equation}
p_V = \sigma\!\left( b_V + \sum_{W \in \Pa^{\G'}(V)} w_{VW}\,(2W - 1) \right),
\end{equation}
where $\sigma$ is the logistic function. Unlike the symmetric form, this family depends on \emph{which} parents are active and is non-additive, exercising asymmetric dependence.

Figure~\ref{fig:nonsym} repeats the density sweep of Test~2 (Figure~\ref{fig:test2res}(a)) under this non-symmetric SEM, reporting the fraction of true edges and true non-edges recovered as a function of graph density; the structure-learning algorithm is unchanged, since it makes no parametric assumptions on the observables.
\begin{figure}[h]
\centering
\includegraphics[width=.4\textwidth]{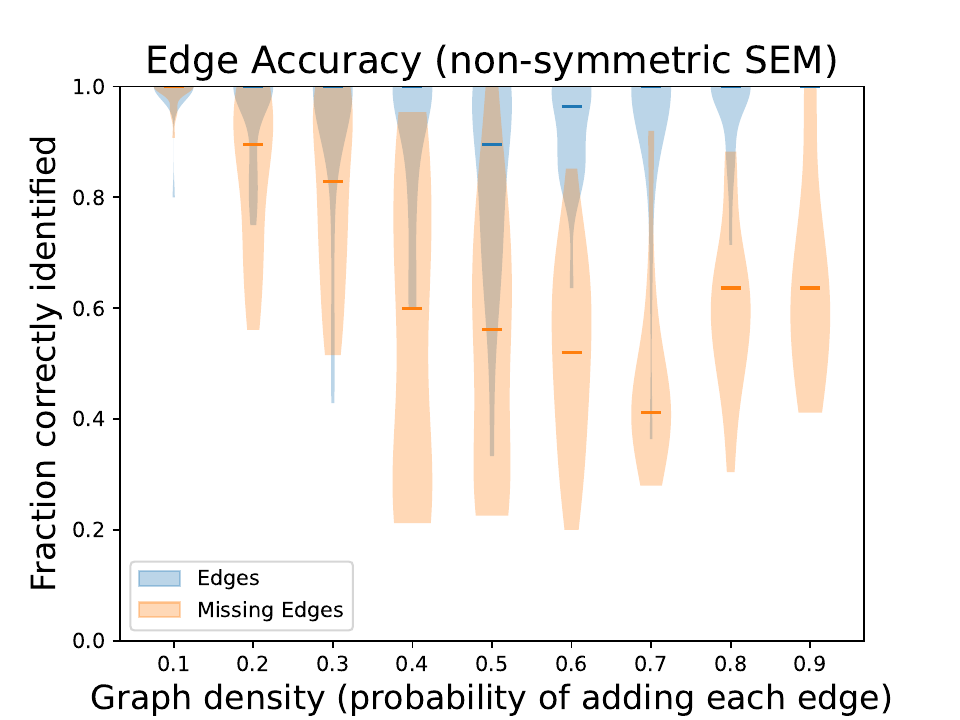}
\caption{Edge-recovery accuracy under the non-symmetric weighted-logistic SEM, as a function of graph density (compare Figure~\ref{fig:test2res}(a) for the symmetric SEM).}\label{fig:nonsym}
\end{figure}

\section{Repeating Test 2 with Incorrect $k$} \label{apx: additional experiments}
We repeated Test 2 with specifications of $k=1$ (too small) and $k=3$ (too large). The results are given in Figure~\ref {fig:test2_wrongk}. Overall, the method does not perform well -- giving close to complete graphs when $k$ is too small and close to empty graphs when $k$ is too large. This makes tuning $k$ rather easy, since a result of a very dense or very sparse graph suggests that a different $k$ should be used.

\begin{figure*}[h]
\centering
\includegraphics[width=.4\textwidth]{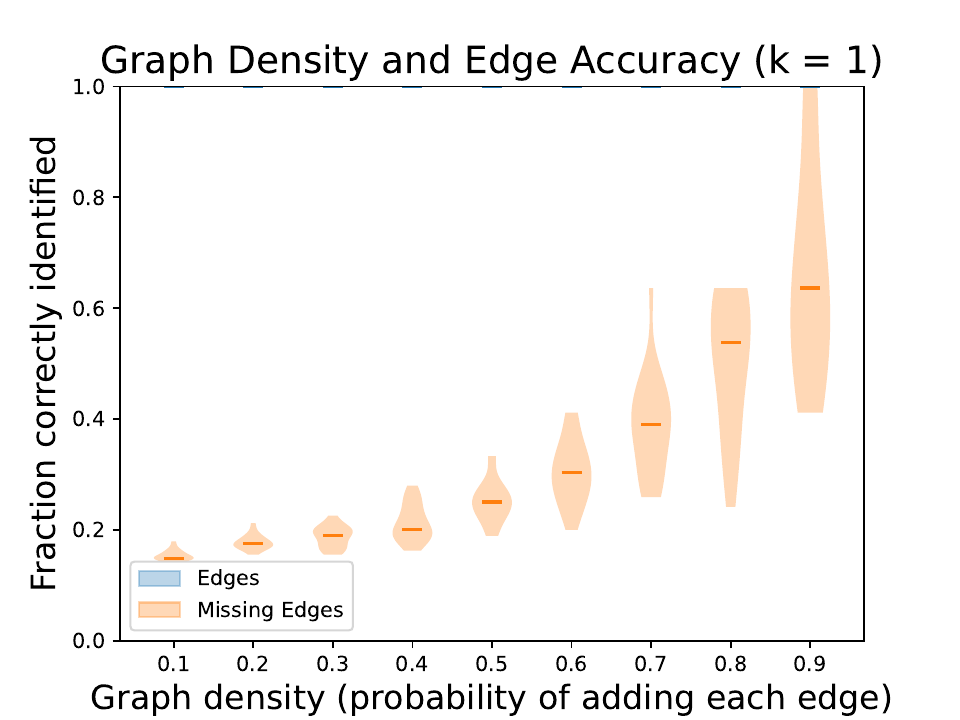}
\includegraphics[width=.4\textwidth]{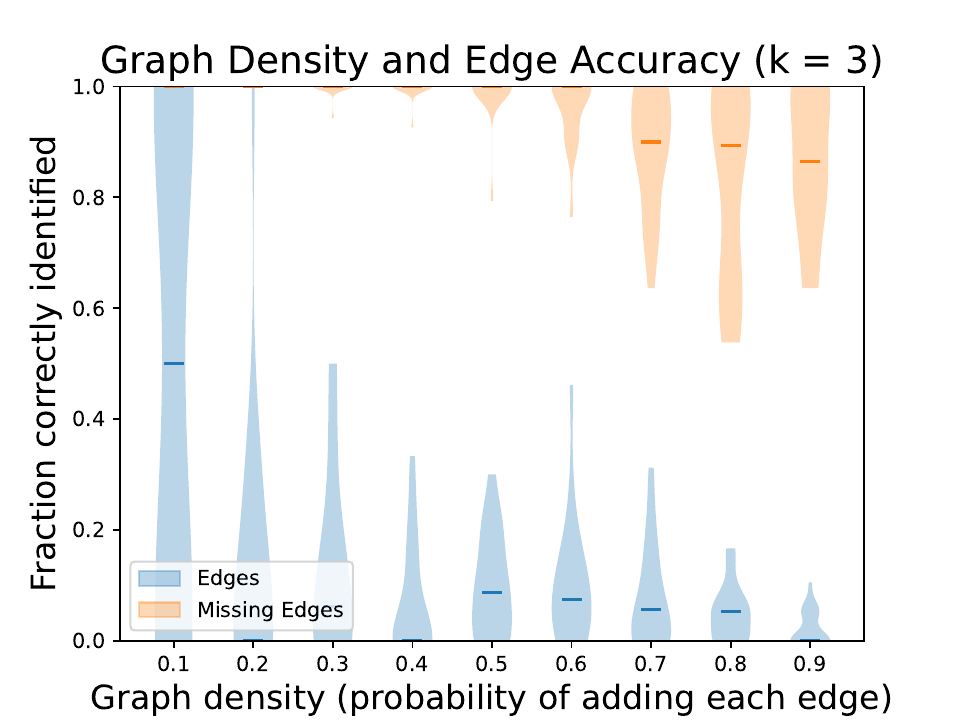}
\caption{The results of Test 2 with incorrectly specified $k$.} \label{fig:test2_wrongk}
\end{figure*}

\section{Deferred Proofs} \label{apx: deferred proofs}

\subsection{Proof of Lemma~\ref{lem:nnrank_dependence}}
\begin{proof}
We will drop the conditioning on $\vec{c}$ in this proof for simplicity. Consider the sum
\begin{equation}
\sigma_j := \sum_{i=1}^j \Pr(u_i) \mat{M}[X, Y \given u_i],
\end{equation}
and note that $\sigma_k = \mat{M}[X, Y]$.  Faithfulness with respect to $\G'$ tells us that there is some assignment, which we call $u_1$, such that $X \not \indep Y \given u_1$. Hence $\nnrank(\sigma_1) > 1$. 

Now,  we show inductively that $\nnrank(\sigma_{i})  = \nnrank(\sigma_{i-1}) + 1$ is a measure $1$ event for $i = 1, \ldots, k$.  Denote $\mat{M}[X, Y \given u_i] = v_iw_i^\top$ with column space $v_i$ drawn from a subspace with non-zero measure on $\R^{n}$. The column space of $\sigma_{i-1}$ is rank $\leq i-1 < m$, so it has measure zero on $\R^{m}$. Hence,  $v_i$ being in the column space of $\sigma_{i-1}$ is a measure $0$ event. We conclude that $\nnrank(\sigma_{i-1} + \mat{M}^{u_i}[X, Y]) = \nnrank(\sigma_{i-1}) + 1$ with measure $1$. Inducting on $i$ gives $\nnrank(\sigma_k) > k$ with measure $1$.
\end{proof}

\suppress{
\begin{proof}
The algorithm designates vertices in $\vec{V}$ into the following sets and succeeds so long as those sets are disjoint.

\begin{enumerate}
\item $\vec{X}_1$ and $\vec{X}_2$
\item $\Nb_2(\vec{X}_1), \Nb_2(\vec{X}_2)$
\item $\vec{T}_{ij}$
\item $\Nb_1(\vec{T}_{ij})$
\end{enumerate}

$\abs{\vec T_{ij}} \geq 2$ and $k \geq 2$, so the number of vertices added to $\vec{X}_1, \vec{X}_2$ in the loop of Algorithm~\ref{alg:kmixprod_construction} is at most $\lceil\lg(2k+2-2) \rceil < \lg(k) + 2$. 

The rest of the sets are dependent on distance 1 and distance 2 neighborhoods of vertices, which depend on the degree of the graph. Recall that our degree bound $\Delta$ on $\G$ only applies to $\vec{V} \setminus \vec{H}$ in $\G_1$ because of FP edges.  For all vertices in $\vec{H}$, we have a degree bound of $\abs{\vec{H}} \leq (2 + \Delta^2) \lceil \lg(k) \rceil$ as given by Lemma~\ref{obs: size of h and degree}.

The vertices in $\vec{T}_{ij}$ clearly might use the larger degree bound --- we must have $V_i, V_j \in \vec{H}$ in order to recover FP edges.  However, by forming sets $\vec{X}_1, \vec{X}_2$ with greedily small distance 2 neighborhoods, we hope that we will have kept them mostly in within a sparse set $\vec{L}$ such that $\abs{\Nb^{\G_1}_{2}(L)} \leq \Delta^2$ for all $L \in \vec{L}$.  We need to ensure that this $\vec{L}$ has enough vertices to form $\vec{X}_1, \vec{X}_2$.

We observe that a vertex $V$ which is at least distance $2$ from $\vec{H}$ contains no $\vec{H}$ vertices in its 1-neighborhood, and thus has a $2$ neighborhood upper bounded by $\Delta^2$.  Call this set $\vec{L}^* \subseteq \vec{L}$.  We know there are no more than $\Delta \abs{\vec{H}}$ vertices that are distance 1 from $\vec{H}$. Hence, we can lower bound the size of $\vec{L}^*$ using the complement of $\vec{H}$ and its $1$-neighborhood.

\begin{equation}
\abs{\vec{L}} \geq \abs{\vec{L}^*} \geq \abs{\vec{V}} - \Delta \abs{\vec{H}}
\end{equation}

$\vec{T}_{ij}, \vec{X}_1, \vec{X}_2$ could have neighborhoods that encroach on $\vec{L}^*$, but such encroachments always do so with bounded degree.  The sizes of their relevant neighborhoods are
\begin{equation}
\begin{aligned}
\abs{\vec{L}^* \cap (\Nb_2(\vec{X}_1) \cup  \Nb_2(\vec{X}_2))} &\leq \Delta^2(\lg(k) + 2)\\
\abs{\vec{L}^* \cap \Nb_1(\vec{T}_{ij}) \cup \vec{T}_{ij}} &\leq 2\Delta^2.
\end{aligned}
\end{equation}
The second inequality comes from the fact that no intersection of $\Nb_1(\vec{T}_{ij})$ with $\vec{L}$ can come from the neighborhood of $V_i$ or $V_j \in \vec{H}$, since we have limited $L$ to being distance 2 from $\vec{H}$.

Putting all of this together, the size of $\vec{V}$ needed to guarantee large enough $\vec{L}^*$ is
\begin{equation}
\vec{V} \geq \underbrace{\Delta \abs{\vec{H}}}_{\text{not a part of }\vec{L}^*} + \abs{\vec{X}_1 \cup \vec{X}_2} + \abs{\vec{L}^* \cap (\Nb_2(\vec{X}_1) \cup  \Nb_2(\vec{X}_2))} +  \abs{\vec{L}^* \cap \Nb_1(\vec{T}_{ij}) \cup \vec{T}_{ij}}.
\end{equation}
The bound is dominated by the first term, which is $\mathcal{O}(\Delta^3 \log(k))$.
\end{proof}}

\subsection{Proof of Lemma~\ref{lem: sepset exists if we have enough vertices}} \label{apx: formal treatment of early vertices}
\begin{proof}
We can form $\vec{S}_i^+$ out of $V_i$ and $\lceil \lg(k + 1) \rceil - 1$ arbitrary other vertices from $\vec{A}_{ij}$. Now, consider subgraph on $\vec{A}_{ij}$, denoted $\G_1[\vec{A}_{ij}]$. Notice that d-separating other vertices in $\vec{A}_{ij}$ from $\vec{S}_i^+$ does not require conditioning on any descendants of $V_i$ or $V_j$. As such, we define a separating set on this subgraph,
\begin{equation}
\vec{C}:= \Mb^{\G[\vec{A}_{ij}]}(\vec{S}_i^+).  
\end{equation}
Since we know $\abs{C} \leq \Delta^2 \lceil \lg(k+1) \rceil$, we need at least $\lceil \lg(k + 1) \rceil-1$ vertices in $\vec{A}_{ij}$ left to join with $V_j$ and make $\vec{S}_j^+$. In summary, we have:
\begin{itemize}
    \item $\lceil \lg(k + 1) \rceil-1$ vertices to augment $\vec S_i^+$.
    \item $\lceil \lg(k + 1) \rceil-1$ vertices to augment $\vec S_j^+$.
    \item $\Delta^2 \lceil \lg(k+1) \rceil$ vertices to d-separate $S_i^+$ and $S_j^+$.
\end{itemize}
The sum of these three completes the proof, since $\vec S_i^+$ was chosen arbitrarily and $\vec S_j^+$ was formed with the leftover vertices.
\end{proof}

\subsection{Proof of Lemma~\ref{obs: size of h and degree}}
Let $a = (2 + \Delta^2) \lceil \lg(k + 1) \rceil - 2$ be equal to the number of non-descendants needed to define $\vec{H}$ (from Definition~\ref{def: early vertices}). Now, we need to show that there are no more than $a$ vertices with fewer than $a$ non-descendants. Index $\vec{V}$ according to a topological ordering and and observe that vertices $V_1, \ldots, V_a$ are non-descendants of vertices $V_{a+1}, \ldots V_{\abs{\vec{V}}}$
. Hence $\vec{H} \subseteq \{V_1, \ldots, V_a\}$ meaning that the cardinality of $\vec{X}$ is no more than $a$.

\subsection{Proof of Lemma~\ref{lem: all fp in H}}

\begin{proof}
A convenient consequence of Lemma~\ref{lem: sepset exists if we have enough vertices} is that it guarantees the existence of IPAs everywhere except within a small subset of vertices.  Let $\overline{\De}(V) := \vec{V} \setminus \De(V)$ be the ``non-descendants'' of $V$. Note that $\vec{A}_{ij} = \overline{\De}(V_i) \cup \overline{\De}(V_j)$. This implies that
\begin{equation} \label{eq: bound A using DE}
\abs{\vec{A}_{ij}}  \geq \max(\abs{\overline{\De}(V_i)}, \abs{\overline{\De}(V_j)}).
\end{equation}
Hence, so long as at least one vertex has enough non-descendants, $\vec{A}_{ij}$ will be large enough to form an IPA.  This set of vertices with enough non-descendants corresponds to the complement of the early vertices.
\end{proof}

\subsection{Proof of Lemma~\ref{lem: tij has separating set}}
\begin{proof}
$\vec{T}_{ij}$ contains both $\Pa(V_i)$ and $\Pa(V_j)$, so Lemma~\ref{lem:parents d-sep} tells us that we contain a separating set.
\end{proof}

\end{document}